\documentclass[10pt,journal,compsoc]{IEEEtran}
\usepackage{xcolor,soul,framed} 
\usepackage{ragged2e} 
\usepackage{algorithm}
\usepackage{algpseudocode}

\usepackage[pdftex]{graphicx}
\graphicspath{{../pdf/}{../jpeg/}}
\DeclareGraphicsExtensions{.pdf,.jpeg,.png}

\usepackage[cmex10]{amsmath}
\usepackage{amsthm}
\usepackage{amssymb} 
\usepackage{thmtools, thm-restate}
\usepackage{array}
\usepackage{mdwmath}
\usepackage{bbm}
\usepackage{mdwtab}
\usepackage{eqparbox}

\newtheorem{lemma}{Lemma}[section]
\usepackage{url}
\usepackage{doi} 

\usepackage{pifont} 
\newcommand{\cmark}{\ding{51}}%
\newcommand{\xmark}{\ding{55}}%
\usepackage[numbers,sort&compress,comma]{natbib}
\usepackage{multirow} 

\usepackage[english]{cleveref}
\renewcommand{\figureautorefname}{Figure}

\def\equationautorefname~#1\null{Eqn.~(#1)\null}
\def\figureautorefname~#1\null{Fig.~#1\null}
\def\lemmaautorefname~#1\null{Lemma~#1\null}

\def\CHK#1 {\textcolor{magenta}{{\bf [CHK:}~#1{\bf ]}}~}
\def\ADD#1 {\textcolor{cyan}{{\bf [ADD:}~#1{\bf}]}~}

\begin{document}
\title{Video DataFlywheel: Resolving the Impossible Data Trinity in Video-Language Understanding}

\author{
    Xiao~Wang, 
    Jianlong~Wu,~\IEEEmembership{Member,~IEEE,}
    Zijia~Lin,\\
    Fuzheng~Zhang,
    Di~Zhang,
    and Liqiang~Nie~\IEEEmembership{Senior Member,~IEEE,}

    \thanks{$\bullet$ Xiao Wang, Jianlong Wu, and Liqiang Nie are with the School of Computer Science and Technology, Harbin Institute of Technology, Shenzhen 518055, China (e-mail: scz.wangxiao@gmail.com; wujianlong@hit.edu.cn; nieliqiang@gmail.com). Corresponding author: Jianlong Wu and Liqiang Nie.}
    \thanks{$\bullet$ Zijia~Lin, Fuzheng~Zhang, and Di~Zhang are with the Kuaishou Technology, Beijing 100193, China (e-mail: \{linzijia, zhangfuzheng, zhangdi08\}@kuaishou.com).}
}



\IEEEtitleabstractindextext{%
\begin{abstract}

\justifying
Recently, video-language understanding has achieved great success through large-scale pre-training. However, data scarcity remains a prevailing challenge. This study quantitatively reveals an ``impossible trinity'' among data quantity, diversity, and quality in pre-training datasets. 
Recent efforts seek to refine large-scale, diverse ASR datasets compromised by low quality through synthetic annotations.
These methods successfully leverage useful information in multimodal video content (frames, tags, ASR transcripts, etc.) to refine the original annotations.
Nevertheless, they struggle to mitigate noise within synthetic annotations and lack scalability as the dataset size expands.
To address these issues, we introduce the Video DataFlywheel framework, which iteratively refines video annotations with improved noise control methods.
For iterative refinement, we first leverage a video-language model to generate synthetic annotations, resulting in a refined dataset. Then, we pre-train on it and fine-tune on human refinement examples for a stronger model. These processes are repeated for continuous improvement.
For noise control, we present AdaTaiLr, a novel noise control method that requires weaker assumptions on noise distribution, thereby proving more effective in large datasets with theoretical guarantees. 
The combination of iterative refinement and AdaTaiLr can achieve better scalability in video-language understanding. 
Extensive experiments show that our framework outperforms existing data refinement baselines, delivering a 3\% performance boost and improving dataset quality with minimal diversity loss. Furthermore, our refined dataset facilitates significant improvements in various video-language understanding tasks, including video question answering and text-video retrieval.

\end{abstract}

}

\maketitle

\ifCLASSOPTIONpeerreview
\begin{center} \bfseries EDICS Category: 3-BBND \end{center}
\fi
%


\section{Introduction}
\IEEEPARstart{I}{n} general video-language understanding, the models are first pre-trained using numerous video-text pairs, and then fine-tuned with minimal task-specific data. This pipeline has demonstrated effectiveness in various tasks, including video-text retrieval \cite{li_umt_2023}, video question answering \cite{xu_pllava_2024}, and text-to-video generation \cite{chen_panda_70m_2024}.

\begin{figure}[t]
    \centering
    \includegraphics[width=0.99\linewidth]{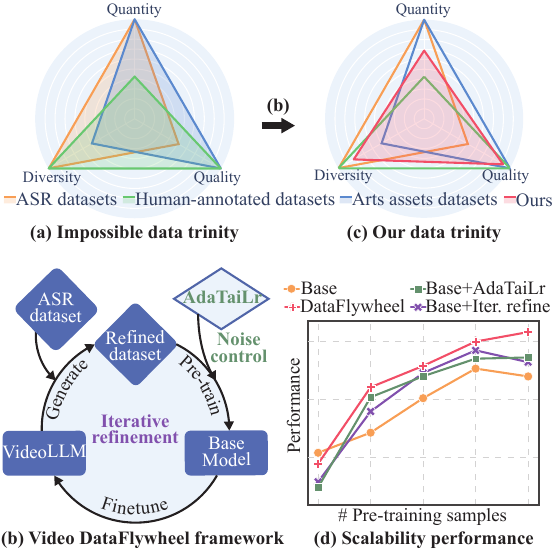}
    \caption{
    For the impossible data trinity (a) among video-language pre-training datasets, we propose the Video DataFlywheel (b) for data refinement. It achieves better trinity (c) and scalability (d) in large data.
    }
    \label{fig:intro}
\end{figure}

However, data scarcity presents a persistent challenge in video-language pre-training \cite{miech_howto100m_2019, bain_webvid_2021, zellers_yt-tmp-180m_2021}. Recent studies suggest that merely increasing data volume may even harm downstream performance \cite{luo_clip4clip_2022, xue_clip_vip_2023, dong_m2paap_2024}. 
Through quantitative analysis, we delve deeply into data issues and identify the \textbf{impossible trinity} within existing datasets. As illustrated in \autoref{fig:intro}(a), datasets from existing curation approaches, including human-annotated, art assert, and Automatic Speech Recognition (ASR) datasets, cannot simultaneously achieve high data quantity, diversity, and quality.

Existing methods addressing the impossible trinity primarily focus on refining the annotations of ASR datasets, as these datasets are of high quantity and diversity but suffer from low quality.
By leveraging foundation models like Large Language Models (LLM) \cite{yang_just_ask_2022, chen_vast_2023, wang_internvid_2023} and Vision Language Models (VLM) \cite{xue_clip_vip_2023, chen_vast_2023, chen_panda_70m_2024, wang_internvid_2023}, these approaches generate synthetic annotations using multimodal video content including video frames, tags, titles, and ASR transcripts.
While showing promise, these methods encounter two key challenges.
(1) \textit{Lack of scalability}. We find that when the size of the refined dataset increases, the downstream performance may not increase accordingly. This is probably due to the limits of foundation models. Specifically, since foundation models generate all annotations, the quality is constrained by the knowledge in these models and their information extraction capability from the multimodal video content, leading to quick saturation as data size scales. To solve this chicken-egg dilemma, it's crucial to iteratively refine the dataset by pre-training stronger foundation models using refined datasets.
(2) \textit{Difficulty in noise control}. Synthetic annotations contain noise from various sources, including model hallucinations and misleading side information. Current noise reduction techniques in vision-language \cite{miech_mil_nce_2020, huang_ncr_2021, feng_ctpr_2023} rely on assumptions about noise distributions that often misalign with real-world data\footnote{Further discussion is provided in Appendix \ref{sec:sup_noise_baselines}}. Consequently, existing dataset refinement methods typically employ pre-trained text-video retrieval models \cite{chen_panda_70m_2024, wang_internvid_2023} to filter out low-similarity annotations, which diminishs annotation diversity \cite{nguyen_improving_2023} and may not consistently improve performance \cite{wang_internvid_2023}.

In this study, we introduce the \textbf{Vid}eo \textbf{D}ata\textbf{F}lywheel (\textbf{VidDF}) framework as a solution to the impossible data trinity. As shown in \autoref{fig:intro} (b), the VidDF framework iteratively refines text annotations and integrates advanced noise control methods. 
Initially, we employ a VideoLLM to generate synthetic annotations based on multimodal video content, maximizing the use of both video data and Video LLM's knowledge for better annotation.
Next, we pre-train a model using the refined dataset. To reduce noise in refined annotations, we propose AdaTaiLr, a novel noise control method utilizing Total Variation Distance (TVD) \cite{loss_truncation, tailr} as a theoretically more robust distance metric instead of KL divergence. Unlike existing noise control methods in vision-language pre-training \cite{huang_ncr_2021, feng_ctpr_2023}, AdaTaiLr does not require the data distribution to be Gaussian mixtures, but only demands the clean distribution as the primary data component. Further, AdaTaiLr enhances TVD estimation through adaptive adjustment of trade-off hyper-parameters, providing theoretical guarantees and setting it apart from previous noise control methods in language modeling \cite{loss_truncation, tailr, li_ENT_2023}.
Finally, we fine-tune the pre-trained model to learn how to refine multimodal video content for better annotations based on a few human-annotated samples. The fine-tuned VideoLLM is then used for annotation refinement. 
Such iterative refinement enables us to surpass the performance limits of foundation models, ensuring better enhancements as the dataset size scales.

We conduct comprehensive experiments to validate the superiority of our framework. 
We first evaluate the quality of our refined dataset. As depicted in \autoref{fig:intro} (b), our analysis reveals that VidDF breaks the impossible trinity by improving data quality with little diversity compromises. For further quantitative results, we pre-train a model on the refined dataset and perform zero-shot video captioning on MSR-VTT \cite{msrvtt_2016}, MSVD \cite{msvd_2011}, and VATEX \cite{vatex_2019} datasets. Our framework outperforms current data refinement methods by a significant 3.1\%. 
Then, on the effectiveness of the noise control method, our ablation studies confirm that AdaTaiLr consistently outperforms noise control baselines.
On the iterative refinement framework, we find that when we scale the dataset, solely using noise control or iterative refinement led to performance saturation or decline, respectively, as portrayed in \autoref{fig:intro} (d). These results highlight the limitations of foundational models and the noise in synthetic annotations, respectively, and underscore the importance of our collaborative approach to combine noise control with iterative refinement for better scalability in video language pre-training.
Finally, by integrating our refined dataset with existing models, we observe significant performance improvements in video question answering and text-video retrieval, demonstrating the utility of our refined dataset.

In summary, this work contributes in four key aspects:
\begin{itemize}
    \item Our quantitative analysis reveals an impossible trinity among quantity, diversity, and quality in existing video-language pre-training datasets. This insight informs a framework to guide the curation, evaluation, and improvement of future pre-training datasets.
    \item We introduce the VidDF framework, which addresses the impossible trinity with a more scalable approach by iteratively refining ASR datasets. This process leverages a VideoLLM pre-trained on refined datasets from previous iterations and fine-tuned with human-annotated examples.
    \item For noise control during pre-training, we present AdaTaiLr. This novel noise control method utilizes a theoretically more robust objective function, which requires weaker assumptions on noise distribution and proves more effective in large datasets with theoretical guarantees.
    \item Comprehensive experiments validate the VidDF framework's superiority by improving data quality with minimal diversity compromise. AdaTaiLr consistently outperforms noise control baselines and helps VidDF achieve better scalability in video language pre-training, leading to notable improvements in downstream video question answering and text-video retrieval tasks.
\end{itemize}

\section{Related work}

\subsection{Video-Language Datasets} \label{sec:related_vidl_ds}
This study focuses on video-language datasets comprising paired video and text, where the text describes the video content. This choice is motivated by two key factors: 1) contemporary video-language models \cite{chen_vast_2023, chat_univi_2023, videochat2} commonly utilize paired video-text data during pretraining, and 2) well-annotated vision-text data can facilitate the generation of qualified instruction-tuning data (e.g., question-answer pairs), as exemplified by LLaVA \cite{liu_llava_2023}.

Existing datasets can be categorized into three main types based on text annotation sources: ASR datasets, art assets datasets, and human-annotated datasets. Each type presents specific limitations for video-language understanding, forming an impossible trinity, as depicted in \autoref{fig:intro} (a).
\textit{1) ASR datasets} \cite{miech_howto100m_2019, zellers_yt-tmp-180m_2021, xue_hd-vila_2022} are typically sourced from YouTube videos, with ASR transcripts serving as annotations. While these datasets offer diversity and quantity due to YouTube's extensive and free content, the quality is often compromised by ASR transcripts. For instance, a manual evaluation in HowTo100M \cite{miech_howto100m_2019} reveals that about 49\% of annotations lack corresponding content in video.
\textit{2) Art assets datasets} \cite{bain_webvid_2021} are curated from art assets platforms like Shutterstock\footnote{\url{https://www.shutterstock.com/video}}, featuring high quality and quantity annotations by artists worldwide. However, these datasets are limited in domain scope and lack diversity.
\textit{3) Human-annotated datasets} \cite{msrvtt_2016, msvd_2011, vatex_2019, youcook_2016, anet_caption_2017, qvhighlights_2021} are meticulously curated from a broad selection of videos and annotated by multiple human annotators to ensure diversity and quality. Nevertheless, the expensive annotation cost limits their quantity.
We leave the details of the impossible trinity in Appendix \ref{sec:sup_imp_trinity}.

To overcome the impossible trinity, this study aims to enhance the annotation quality of ASR datasets.

\subsection{Refining Video-Language Datasets from Web}

\begin{figure}[t]
    \centering
    \includegraphics[width=0.99\linewidth]{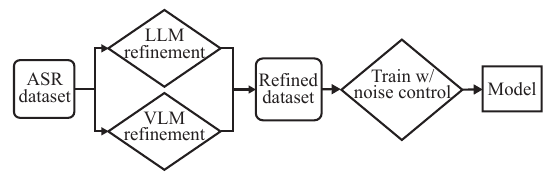}
    \caption{Unified framework of existing dataset refinement methods, consisting of three procedures in diamond boxes.}
    \label{fig:general_refine_pipe}
\end{figure}

\begin{table}[t]

\caption{
Comparison between vision-language dataset refinement methods.
}

\label{tab:related_work_tab}
\resizebox{\columnwidth}{!}{

\begin{tabular}{ccccc}
\hline
Method &
  \begin{tabular}[c]{@{}c@{}}LLM \\ refinement\end{tabular} &
  \begin{tabular}[c]{@{}c@{}}VFM\\ refinement\end{tabular} &
  \begin{tabular}[c]{@{}c@{}}Noise\\ control\end{tabular} &
  \begin{tabular}[c]{@{}c@{}}Iterative\\ refinement\end{tabular} \\ \hline
MIL-NCE \cite{miech_mil_nce_2020}    & \xmark & \xmark & \cmark & \xmark \\
NCR \cite{huang_ncr_2021}            & \xmark & \xmark & \cmark & \xmark \\
Just Ask \cite{yang_just_ask_2022}   & \cmark & \xmark & \xmark & \xmark \\
LaCLIP \cite{fan_laclip_2023}        & \cmark & \xmark & \xmark & \xmark \\
CTPR \cite{feng_ctpr_2023}           & \xmark & \xmark & \cmark & \xmark \\
CLIP-ViP \cite{xue_clip_vip_2023}    & \xmark & \cmark & \xmark & \xmark \\
VAST \cite{chen_vast_2023}           & \cmark & \cmark & \xmark & \xmark \\
Panda70M \cite{chen_panda_70m_2024}  & \cmark & \cmark & \xmark & \xmark \\
InternVid \cite{wang_internvid_2023} & \cmark & \cmark & \xmark & \xmark \\
VeCLIP \cite{lai_veclip_2024}        & \cmark & \cmark & \xmark & \xmark \\ \hline
Ours                                 & \cmark & \cmark & \cmark & \cmark \\ \hline
\end{tabular}

}
\end{table}

Video-language datasets sourced from the web, especially ASR datasets discussed in Section \ref{sec:related_vidl_ds}, commonly exhibit low annotation quality. In this section, we introduce a novel and unified framework for existing refining methods. Each method aligns with one or more of the three procedures outlined in Figure \ref{fig:general_refine_pipe}.

\textit{Large Language Model refinement} leverages LLMs to enhance annotations through text rewriting augmentation \cite{fan_laclip_2023}, content extraction \cite{yang_just_ask_2022}, and integration of various sources such as ASR transcripts \cite{chen_vast_2023}, web text \cite{lai_veclip_2024}, or multimodal captioners \cite{chen_vast_2023, wang_internvid_2023, chen_panda_70m_2024, lai_veclip_2024}. These methods rely on high-quality raw annotations or collaboration with vision models due to the absence of visual perception ability.

\textit{Vision Language Model refinement} utilizes pre-trained caption models to produce visually grounded synthetic annotations \cite{li_blip_2022, betker_dali3_2023, nguyen_improving_2023, xue_clip_vip_2023, chen_vast_2023, chen_panda_70m_2024, wang_internvid_2023, lai_veclip_2024}. However, their quality heavily depends on caption models, leading to a issues like noise and lack of diversity. For instance, Nguyen et al. \cite{nguyen_improving_2023} observed that unfiltered BLIP2 captions perform even worse than raw web captions on the DataComp 128M dataset.

\textit{Noise control} focuses on reducing noise in annotations by assuming certain noise distributions, such as MIL-NCE \cite{miech_mil_nce_2020}, NCR \cite{huang_ncr_2021}, and CTPR \cite{feng_ctpr_2023}. However, these assumptions may not always align with real data distributions (See Appendix \ref{sec:sup_noise_baselines} for details).
%

In this study, we introduce a novel noise control method AdaTaiLr, and propose an iterative refinement framework for better scalability in video-language pre-training. We compare our method with existing video dataset refinement methods in \autoref{tab:related_work_tab}.

\subsection{Video Large Language Models} \label{sec:related_work_vidlm_model}

Video Large Language Models (VideoLLM) typically comprise Vision Foundational Models (VFMs) \cite{clip_2021}, LLMs \cite{vicuna_llm}, and connectors bridging them. VideoLLMs are categorized based on their connectors into concatenation-based, Q-Former-based, and cross-attention-based models.
Concatenation-based models \cite{video_chatgpt_2023, video_llava, llava_next} utilize an MLP on VFM patch embeddings, concatenating the MLP output with LLM token embeddings. While simple and effective, these models are memory-intensive due to the long VFM patch length.
Q-Former-based models \cite{videochat_2023, videochat2} employ a transformer decoder with learnable tokens to compress VFM patch embeddings, reducing memory usage but risking performance loss from information compression.
Cross-attention-based models \cite{videococa_2023} incorporate cross-attention within LLM layers to integrate visual data, requiring large training data due to high parameter complexity.
This study focuses solely on concatenation-based methods to isolate the impact of model architectures.

\section{Data Flywheel for video-language understanding}

\begin{figure*}[t]
    \centering
    \includegraphics[width=0.99\linewidth]{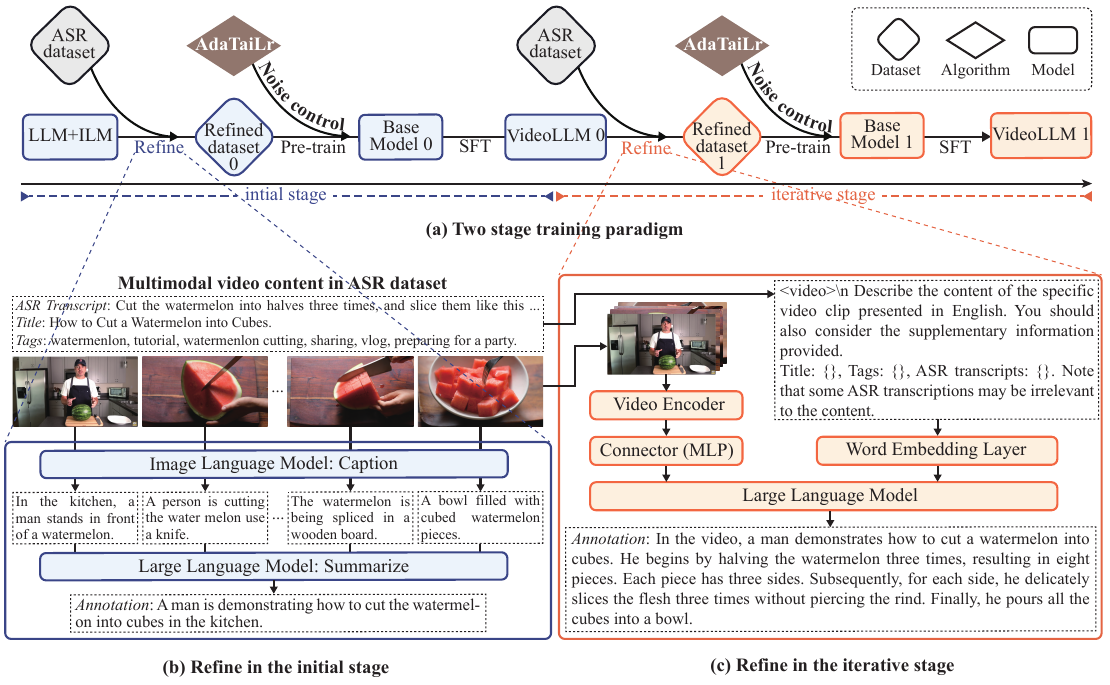}
    \caption{
    Method overview.
    (a) Our video dataflywheel framework comprises two stages. The initial refinement stage refines the ASR dataset by prompting LLM and ILM, since there is no VideoLLM at this stage. The iterative refinement stage refines the dataset using VideoLLM trained in the previous stage. AdaTaiLr is applied for noise control at both stages in pre-training.
    (b) During initial refinement, an LLM summarizes the image captions generated by frames.
    (c) In iterative refinement, a VideoLLM generates annotations based on multi-modal video content. 
    }
    \label{fig:method_overview}
\end{figure*}

\subsection{Overview of VidDF}

To address the challenges of scalability and noise control in dataset refinement, we propose the VidDF framework that iteratively refines dataset annotations with noise control. As illustrated in \autoref{fig:method_overview} (a), our framework comprises two stages. 
We begin with the \textit{initial stage}, where we have only LLM and Image-Language Models (ILM) instead of a VideoLLM. We refine the dataset by prompting LLM and ILM and then use the refined dataset to train our initial VideoLLM for further refinement. 
Thereafter, the \textit{iterative stage} refines the ASR dataset using VideoLLM trained in the previous stage.

Both the initial and iterative stages follow a similar pipeline, based on which this section is organized. 
Specifically, we first refine the ASR dataset in \autoref{sec:anno_refine}. 
For noise control during pre-training, we introduce a novel method AdaTaiLr for both stages in \autoref{sec:noise_control}.
Then, we Pre-Train (PT) a VideoLLM on this refined dataset, and perform Supervised Fine-Tuning (SFT) on human examples of annotation refinement in \autoref{sec:pretrain_sft_detail}. Finally, the resulting VideoLLM is used in the next refinement stage.

\subsection{Annotation Refinement} \label{sec:anno_refine}

We refine the ASR dataset's textual annotations using a VideoLLM (or LLM+ILM), which interprets the multimodal video content's textual and visual cues. 
The rationale of this approach is based on three key characteristics of ASR datasets, as exemplified by the video example in \autoref{fig:method_overview} (b). 
\begin{itemize}
    \item The dataset comprises textual content—titles, tags, and ASR transcripts—alongside visual information extractable from video frames via ILM captioners or video encoders.
    \item The visual and textual elements are mutually informative. For instance, ASR transcripts can summarize visual content with phrases like "three times" during a cutting procedure, while visuals offer extra details beyond "like this" in a slicing procedure.
    \item The textual content often contains irrelevant information, such as "sharing" and "vlog" in tags or personal feelings in ASR transcripts.
\end{itemize}
Given the characteristics above, an LLM or VideoLLM is vital, as its extensive knowledge enables the integration of textual and visual clues for more precise and comprehensive video annotations.
The detailed refining process varies between the initial and refinement stages, which will be detailed below.

\subsubsection{Annotation Refinement at the Initial Stage}

At the initial stage, prior to the training of VideoLLMs, we utilize both LLM and ILM for annotation refinement. Intuitively, the ILM extracts visual information to the maximum extent of current model capabilities, while the LLM integtates key visual and textual elements.
As depicted in \autoref{fig:method_overview} (b), this stage involves two steps: captioning and summarization. For captioning, we sample frames uniformly and describe each using pre-trained BLIP2 \cite{li_blip2_2023}. For summarization, we leverage Vicuna \cite{vicuna_llm} to consolidate frame captions into a cohesive paragraph. 
Nevertheless, this stage also presents problems, including caption noise from ILM and the absence of video understanding capability. These issues will be addressed in \autoref{sec:noise_control} and the subsequent section, respectively.

\subsubsection{Annotation Refinement at the Iterative Stage}

At the iterative stage, we refine the ASR dataset using VideoLLM trained in the previous stage. The rationale is that VideoLLM, equipped with the capabilities of both ILM and LLM, also possesses the additional ability to understand videos.
As illustrated in \autoref{fig:method_overview} (c), following LLaVA \cite{liu_llava_2023}, we adopt a simple VideoLLM with all necessary components:
\begin{itemize}
    \item Video Encoder. We adopt TimeSformer-L \cite{bertasius_timesformer_2021}. It is a ViT-L/14 with extra self-attention before each transformer layer, focusing on the temporal dimension. The ViT is initialized from CLIP, while the temporal attention is trained from scratch. A zero-initialized fully connected layer is added after the temporal attention layer to smooth the training.
    \item Connector. Following LLaVA \cite{liu_llava_2023}, we use a two-layer perceptron to project video features from the 2-nd last layer into language space.
    \item LLM. We use Vicuna-1.5-7B \cite{vicuna_llm} as the large language model.
\end{itemize}
We use the prompt in \autoref{fig:method_overview} (c) to generate annotations by integrating both textual and visual information.

\subsection{AdaTaiLr: Noise Control for Pre-training} \label{sec:noise_control}

Similar to LLMs in natural language processing, existing VideoLLMs are generally trained to minimize the KL Divergence (KLD) between predicted and real data distributions.
However, recent studies \cite{loss_truncation} suggest that KLD is sensitive to data noise, which is more prevalent in video-language data than language data.
To address noise in natural language processing, some researchers have explored leveraging total variation distance as a more robust metric \cite{loss_truncation, tailr}. Nevertheless, when applied to video-language pre-training, these methods such as TaiLr \cite{tailr} suffer from high estimation errors of TVD, because it can only be calculated through statistical estimation.
In this section, we propose \textbf{Ada}ptive \textbf{TaiLr} (\textbf{AdaTaiLr}), which offers improved TVD optimization with theoretical guarantees.

\subsubsection{Preliminaries for KLD, TVD, and TaiLr}

VideoLLMs are generally formulated as a conditional language generation task: given video-language context $\mathbf{x}$, a conditional generative model $p_\theta(\mathbf{y}|\mathbf{x})$ parameterized by $\theta$ is required to generate target text sequence $\mathbf{y}=(y_1, ..., y_T)$. 
Traditional training objectives minimize the KLD between predicted distribution $p_\theta$ and real data distribution $p_o$:
\begin{equation}
    \mathcal{L}_{\textmd{KLD}} = - \mathbb{E}_{\mathbf{y}\sim p_o}
    \left [
        \sum_{t=1}^T\textmd{log}p_\theta(y_t|\mathbf{y}_{<t},\mathbf{x})
    \right ]
     - H(p_o),
\end{equation}
where $H(p_o)$ is the entropy of the real data distribution $p_o$, which is often omitted during calculation since it is a constant with respect to $\theta$.

Because KLD is sensitive to noise in the training data \cite{loss_truncation} and suffers from mismatch to evaluation metric \cite{gold}, TaiLr \cite{tailr} introduces TVD from probability theory \cite{tvd_textbook} as a robust alternative to KLD:
\begin{equation}
    \mathcal{L}_{\textmd{TVD}} (p_o, p_\theta) = \frac{1}{2} \sum_{\mathbf{y}\in\mathcal{Y}}{\left | 
        p_o(\mathbf{y}|\mathbf{x}) - p_{\theta}(\mathbf{y}|\mathbf{x})
    \right |},
\end{equation}
where $\mathcal{Y}$ is the space of all possible text sequences. 
Intuitively, minimizing the $L_1$-norm of $p_o-p_\theta$ will make the model find a sparse solution \cite{ESL} of the probability distribution $p_\theta$. In other words, probability is allocated to the major part of the real data distribution, ignoring the outliers which are probably the noise.

Since directly calculating TVD by enumerating the whole $\mathcal{Y}$ space is impractical, Ji \textit{et al.} \cite{tailr} proposes to minimize the estimated upper bound of TVD using the TaiLr loss:
\begin{equation}
    \mathcal{L}_{\textmd{TaiLr}} = 
    \mathbb{E}_{\mathbf{y}\sim p_o}
    \left [
        - \sum_{t=1}^T
            \frac{p_\theta^{<t}(y_t)}{\gamma + (1 - \gamma)p_\theta^{<t}(y_t)} 
        \textmd{log}p_\theta^{<t}(y_t)
    \right ],
\end{equation}
where $p_\theta^{<t}(y_t)$ denotes $p_\theta^{<t}(y_t|\textbf{y}_{<t}, \textbf{x})$ for simplicity, and $\gamma\in[0,1]$ is a trade-off constant for the estimation error of the upper bound of TVD $\epsilon(p_o^{<t}, p_\theta^{<t})$:
\begin{equation} \label{eq:estimation_error}
    \epsilon(p_o^{<t}, p_\theta^{<t}, \gamma) = (1 - \gamma) \mathcal{L}_{\textmd{TVD}}(p_o^{<t}, p_\theta^{<t}) + 
    \gamma 2H_2(p_o^{<t}),
\end{equation}
where $H_\alpha(p)$ is the Tsallis $\alpha$-entropy:
\begin{equation}
    H_\alpha(p) = \left\{\begin{matrix}
     \frac{1}{\alpha(\alpha-1)}(1 - \sum_ip_i^\alpha),  &  \alpha\ne 1, \\
     -\sum_ip_i\textmd{log}p_i,  &  \alpha=1.
    \end{matrix}\right.
\end{equation}

In this paper, we find that a constant $\gamma$ in TaiLr is sub-optimal, as it does not necessarily minimize the estimation error in \autoref{eq:estimation_error}. 
This issue is particularly pronounced in video-language understanding due to the significant fluctuations of $H_2(p_o^{<t})$, caused by the larger sequence space of conditional probabilities $p_o^{<t}$ from diverse video inputs.
Thus, we propose an adaptive function to adjust $\gamma$ automatically.  We will show that our method surpasses TaiLr with theoretical guarantees in the next section.

\begin{algorithm}[t]
\caption{The Pseudo-code of AdaTaiLr}\label{alg:adatailr_pseudo}
\begin{algorithmic}
\Require \\
$\mathbf{P}\in\mathbb{R}^{L\times N}$: VideoLLM output, where $P_{ij}$ is the probability of the $i$-th token being ID $j$ in the vocabulary, $L$ is sequence length, and $N$ is vocabulary size \\
$\mathbf{y}\in\mathbb{R}^L$: Label, where $y_i$ is the ground-truth ID of the $i$-th token \\
$\lambda\in\mathbb{R}$: Hyper-parameter in \autoref{eq:Gamma_appo} 
\Ensure $\mathcal{L}_{\textmd{AdaTaiLr}}$: AdaTaiLr loss in \autoref{eq:AdaTaiLr}
\end{algorithmic}
\begin{algorithmic}[1]
\For{$i \gets 1$ to $L$}
    \State $t_i \gets \frac{1}{2} \left( |1-p_{iy_i}| + \sum_{j=1,j\ne y_i}^Np_{ij} \right )$
    \State $h_i \gets \frac{1}{2} \left( 1-\sum_{j=1}^Np_{ij}^2 \right )$
    \State $\gamma_i \gets \frac{1}{2} + \lambda(t_i - 2h_i)$
\EndFor
\State $\mathcal{L}_{\textmd{AdaTaiLr}} \gets - \sum_{i=1}^L \frac{p_{iy_i}}{\gamma_i+(1-\gamma_i)p_{iy_i}}\textmd{log}p_{iy_i}$
\end{algorithmic}
\end{algorithm}

\subsubsection{AdaTaiLr} \label{sec:adatailr}

Intuitively, term $\mathcal{L}_{\textmd{TVD}}(p_o^{<t}, p_\theta^{<t})$ in the TaiLr estimation error \autoref{eq:estimation_error} will change during training. Thus, the optimal trade-off of $\gamma$ will change accordingly. Therefore, we can find a function instead of a constant for $\gamma$ that minimizes the estimation error, as indicated in the theorem below.

\begin{restatable}[Optimal $\gamma$]{theorem}{og}
\label{thm:og}
Given a VideoLLM model $p_\theta^{<t}(y_t|\textbf{y}_{<t}, \textbf{x})$ parameterized by $\theta$ and the real data distribution $p_o^{<t}(y_t|\textbf{y}_{<t}, \textbf{x})$. The following function:
\begin{equation} \label{eq:opt_gamma}
    \Gamma_{opt}(p_o^{<t}, p_\theta^{<t}) = \mathbbm{1} \left [
        \mathcal{L}_{\textmd{TVD}}(p_o^{<t}, p_\theta^{<t}) - 2H_2(p_o^{<t})
    \right ],
\end{equation}
where where $\mathbbm{1}[z]$ is the indicator function:
\begin{equation} \label{eq:indicator_func}
    \mathbbm{1}[z]=\begin{cases}
     1, z\ge 0, \\ 
     0, z<0,
    \end{cases}
\end{equation}
minimizes the upper bound of TaiLr estimation error $\epsilon$:
\begin{equation}
    \Gamma_{opt}(p_o^{<t}, p_\theta^{<t}) = \min_{\gamma} \epsilon(p_o^{<t}, p_\theta^{<t}, \gamma).
\end{equation}
\end{restatable}
We leave the full proof in Appendix \ref{sec:proof_og}. 

In the experimental calculations, the above optimal $\Gamma_{opt}$ has two issues. One is that the indicator function \autoref{eq:indicator_func} is not smooth and thus sensitive to noise. The other is that \autoref{eq:opt_gamma} contains the real data distribution $p_o$ which is unavailable during training. To solve these issues, we can get the following approximation theorem by using a smooth approximation of the indicator function, and the predicted distribution to approximate the real distribution.

\begin{restatable}[Approximation of Optimal $\gamma$]{theorem}{aog}
\label{thm:aog}
Assume that after some warm-up steps during training, there exists $D>0$ under which $\left \| p_\theta - p_o \right \|_1 \le 2D$.
Given one-hot distribution sampled from real data $e^{(w)}\sim p_o^{<t}$, the following function:
\begin{equation}
    \tilde{\Gamma}_{opt}(p_o^{<t}, p_\theta^{<t})=\frac{1}{2} + \lambda \left (
        \mathcal{L}_{\textmd{TVD}}(e^{(w)}, p_\theta^{<t}) - 2H_2(p_\theta^{<t})
    \right ),
\end{equation}
achieves the following approximation guarantee towards $\Gamma_{opt}$.
\begin{equation}
    \mathbb{E}_{w\sim p_o} \left [ \left |  
        \epsilon(p_o^{<t}, p_\theta^{<t}, \tilde{\Gamma}_{opt}) - \epsilon(p_o^{<t}, p_\theta^{<t}, \Gamma_{opt})
    \right | \right ]
    \le
    \frac{a}{\lambda} + bD,
\end{equation}
where $\lambda>0$ is a constant controlling the smoothness of the approximation of the indicator function, and $a$, $b$ are constants depending on the relationship of $\lambda$ and $D$, $|a|<\frac{9}{16}$ and $|b|<4$. 
\end{restatable}
The full proof is presented in Appendix \ref{sec:proof_aog}.

Therefore, in our experiment, we use the AdaTaiLr loss:
\begin{equation} \label{eq:AdaTaiLr}
    \mathcal{L}_{\textmd{AdaTaiLr}} = 
    \mathbb{E}_{\mathbf{y}\sim p_o}
    \left [
        - \sum_{t=1}^T
            \frac{p_\theta^{<t}(y_t)}{\Gamma + (1 - \Gamma)p_\theta^{<t}(y_t)} 
        \textmd{log}p_\theta^{<t}(y_t)
    \right ],
\end{equation}
where
\begin{equation} \label{eq:Gamma_appo}
    \Gamma = \frac{1}{2} + \lambda \left (
        \mathcal{L}_{\textmd{TVD}}(e^{(y_t)}, p_\theta^{<t}) - 2H_2(p_\theta^{<t})
    \right ),
\end{equation}
and $\lambda$ is a fixed constant. During training, we clamp $\Gamma$ between $[0, 1]$. To counter the negative effect of random prediction at the early training stage, we set a threshold $\delta$ as the lower bound of the weighting factor before $\textmd{log}p_\theta^{<t}(y_t)$. We provide a pseudo code of AdaTaiLr in Algorithm \autoref{alg:adatailr_pseudo}.

\subsection{Pre-training and Supervised Fine-tuning} \label{sec:pretrain_sft_detail}

The pre-training and supervised fine-tuning procedures are the same in both the initial and interactive stages.

During pre-training, we optimize the VideoLLM to generate target textual annotations from video inputs using the AdaTaiLr loss. The video-text pairs used for this purpose are sourced from the refined dataset from \autoref{sec:anno_refine}.

For supervised fine-tuning, we optimize the VideoLLM using the standard LM loss since the datasets in this stage are of high quality. We employ two datasets: VideoChatGPT-100K \cite{video_chatgpt_2023} for general video instructional following and asrRefine-10K for ASR dataset refinement capabilities. We curate the asrRefine-10K dataset by randomly selecting 10K video-caption pairs from the training splits of video captioning datasets Youcook2 \cite{youcook_2016}, ActivityNet-Captions \cite{anet_caption_2017}, and QV-Highlights \cite{qvhighlights_2021}, along with their ASR transcripts, titles, and tags. We construct question-answer pairs using the prompt format shown in \autoref{fig:method_overview} (c), with human captions serving as the answers. To prevent data leakage, we do not select videos from our evaluation datasets MSR-VTT \cite{msrvtt_2016}, MSVD \cite{msvd_2011}, and VATEX \cite{vatex_2019}.

\section{Experiments}

This section is organized as follows. In \autoref{sec:exp_dataflywheel}, we conducted experiments on ASR dataset refinement, including baseline comparison, ablation studies, and sensitivity analysis.
To further understand the superiority of the VidDF framework, we performed an in-depth analysis in \autoref{sec:in_depth_analysis}.
Finally, we integrated our refined dataset with existing models to validate its improvements in video question answering (\autoref{sec:videoqa_exps}) and text-video retrieval (\autoref{sec:videoret_exps}).

\subsection{Main Results of DataFlywheel} \label{sec:exp_dataflywheel}

\subsubsection{Experimental Settings} \label{sec:dataflywheel_exp_settings}

\textbf{Quantify the Data Trinity}. We define \textit{Quantity} as the text annotations collected within a set budget, \textit{Diversity} as the number of unique tokens in the dataset, and \textit{Quality} as annotation accuracy. For precise definitions and calculation details, refer to Appendix \ref{sec:sup_imp_trinity}.

\noindent\textbf{ASR dataset and the initial stage}. 
Since our contributions focus on iterative refinement, we kept the initial refinement the same as InternVid \cite{wang_internvid_2023}, and used InternVid-10M-FLT \cite{wang_internvid_2023} as the ``Refined dataset 0'' in \autoref{fig:method_overview} (a). We filtered out clips shorter than 2 seconds and got 6,665,285 clips downloaded successfully. The number of clips used in pre-training varies among experiments. Unless otherwise specified, we used 770K videos to ensure consistency with other baselines such as Video-LLaVA \cite{video_llava}.

\noindent\textbf{Models and implementation details}
During pre-training, the ViT backbone of the video encoder and the LLM are frozen, with only the temporal attention layers and the connector being fine-tuned. We set batch size and learning rate to be 256 and 1e-3, and scaled the learning rate of the temporal attention module by 0.1 to stabilize training.
During SFT, the entire video encoder is frozen, with only the connector and LLM being fine-tuned. We set batch size and learning rate to be 128 and 2e-5.
Our model was trained using 16 A800 GPUs.
%


\noindent\textbf{Evaluation datasets and metrics.} 
We adopted different datasets and metrics for evaluating pre-training and SFT models. 
For pre-training, we evaluated models on video captioning datasets MSR-VTT \cite{msrvtt_2016}, MSVD \cite{msvd_2011}, and VATEX \cite{vatex_2019} using CIDEr \cite{vedantam_cider_2015} metric. This stage focuses on aligning vision and text modalities, and thus the video captioning dataset can effectively assess the quality of pre-trained models. 
%
For SFT, we evaluated models on traditional open-ended Video Question Answer (VideoQA) benchmarks including MSVD-QA \cite{msvd_2011}, MSRVTT-QA \cite{msrvtt_2016}, ActivityNet-QA \cite{anet_caption_2017}, and TGIF-QA \cite{li_tgifqa_2016}. Considering that ground-truth answers in these benchmarks are single-word, we followed Maaz et al., \cite{video_chatgpt_2023} to prompt GPT-3.5 for evaluating the accuracy. Note that we adopted \texttt{GPT-3.5-turbo-0125} since the earlier versions will be deprecated soon. 

\begin{table*}[t]

\caption{
Relationship between PT and SFT evaluation results.
}

\label{tab:pt_sft_eval_rel}
\centering

\begin{tabular}{ccccccccccc}
\hline
\multirow{3}{*}{\# PT Data} & \multicolumn{4}{c}{PT Evaluation} &  & \multicolumn{5}{c}{SFT Evaluation}        \\ \cline{2-5} \cline{7-11} 
                            & MSRVTT  & MSVD   & VATEX  & Mean  &  & MSVD & MSRVTT & ActivityNet & TGIF & Mean \\ \cline{2-5} \cline{7-11} 
                            & \multicolumn{4}{c}{Caption CIDEr} &  & \multicolumn{5}{c}{QA Accuracy}           \\ \hline
85K                         & 60.2    & 136.4  & 57.5   & 84.7  &  & 71.4 & 58.3   & 46.5        & 70.4 & 61.6 \\
256K                        & 61.2    & 141.5  & 60.6   & 87.8  &  & 72.7 & 60.7   & 47.5        & 70.4 & 62.8 \\
770K                        & 62.1    & 148.4  & 62.6   & 91.0  &  & 72.9 & 61.3   & 47.8        & 71.8 & 63.5 \\
2310K                       & 62.8    & 151.0  & 63.9   & 92.6  &  & 72.6 & 61.4   & 49.7        & 72.2 & 64.0 \\ \hline
\end{tabular}
\end{table*}

\subsubsection{Relationship between PT \& SFT Evaluation Results}

The whole training process consists of PT and SFT stages, which is time-consuming. If an early estimation method for final SFT performance existed, it could significantly reduce the required time. 
To investigate this, we varied the amount of video PT data (using “Refined dataset 0”) and recorded both PT and SFT evaluation results. As illustrated in \autoref{tab:pt_sft_eval_rel}, the mean SFT evaluation results correlate with the mean PT evaluation results across a wide range of PT data. Additionally, the performance improvements observed in each dataset are consistent.
These observations suggest that improvements in PT evaluation are indicative of improvements in SFT evaluation. Consequently, for most experiments in this subsection, we reported only the PT evaluation results.

\subsubsection{Comparison with Data Refinement Baselines}

\begin{table}[t]

\caption{
Comparison with other refined datasets.
}

\label{tab:data_refine_comp}
\centering
\resizebox{\linewidth}{!}{

\begin{tabular}{ccccc}
\hline
                                     & MSRVTT  & MSVD   & VATEX  & Mean  \\ \cline{2-5} 
\multirow{-2}{*}{Dataset}            & \multicolumn{4}{c}{Caption CIDEr} \\ \hline
{\color[HTML]{1F2329} Valley \cite{luo_valley_2023}}    & 59.5          & 142.0          & 58.5          & 86.6          \\
VAST \cite{chen_vast_2023}           & 59.8    & 138.0  & 56.5   & 84.8  \\
{\color[HTML]{1F2329} InternVid-10M-FLT \cite{wang_internvid_2023}} & 61.5          & 146.9          & 62.0          & 90.1          \\
Panda-70M \cite{chen_panda_70m_2024} & 62.2    & 144.1  & 62.0   & 89.4  \\ \hline
\textbf{Ours}                                                       & \textbf{63.6} & \textbf{150.5} & \textbf{64.6} & \textbf{92.9} \\ \hline
\end{tabular}

}
\end{table}

\begin{table}[t]

\caption{
Comparison with other noise control baselines.
}

\label{tab:noise_control_comp}
\centering

\begin{tabular}{ccccc}
\hline
\multirow{2}{*}{Method}                & MSRVTT        & MSVD           & VATEX         & Mean          \\ \cline{2-5} 
                                       & \multicolumn{4}{c}{Caption CIDEr}                              \\ \hline
None                                   & 61.5          & 146.9          & 62.0          & 90.1          \\
Filtering \cite{nguyen_improving_2023} & 61.9          & 148.8          & 60.7          & 90.5          \\
NCR \cite{huang_ncr_2021}              & 61.3          & 148.9          & 61.1          & 90.4          \\
CTPR \cite{feng_ctpr_2023}             & 60.7          & 144.2          & 60.8          & 88.6          \\
Loss Truncation \cite{loss_truncation} & 60.8          & 146.6          & 62.0          & 89.8          \\
TaiLr \cite{tailr}                     & 62.2          & 148.7          & 61.9          & 90.9          \\
ENT \cite{li_ENT_2023}                 & 61.6          & 147.5          & 62.4          & 90.5          \\ \hline
\textbf{AdaTaiLr}                      & \textbf{63.1} & \textbf{150.3} & \textbf{62.6} & \textbf{92.0} \\ \hline
\end{tabular}


\end{table}

To validate the effectiveness of our VidDF framework for dataset refinement, we compared our refined dataset after all two stages with other refined video-language datasets: Valley \cite{luo_valley_2023}, VAST-27M \cite{chen_vast_2023}, InternVid \cite{wang_internvid_2023}, and Panda-70M \cite{chen_panda_70m_2024}. Specifically, we sampled 770K (matching Valley \cite{luo_valley_2023}) pairs from each dataset as the pre-training dataset, while keeping other settings in \autoref{sec:dataflywheel_exp_settings} unchanged.

As illustrated in \autoref{tab:data_refine_comp}, our dataset consistently outperforms better in all PT evaluation datasets, achieving 3.1\% improvements over the current state-of-the-art dataset. Furthermore, experiments with a scaled number of training data show even greater improvements, illustrated in \autoref{fig:ablation_studies}.
Comparing other datasets yields interesting findings. Notably, both VAST \cite{chen_vast_2023} and Panda-70M \cite{chen_panda_70m_2024} are refined based on HD-VILA \cite{xue_hd-vila_2022} dataset, but their performance varies 5\%. There are major differences between them: 1) VAST has no noise control, while Panda-70M filters low-confidence video-text pairs based on a pre-trained video-text retrieval model. 2) Panda-70M integrates pseudo captions from multiple VLMs, enhancing caption diversity. Our dataset also benefits significantly from noise control and diversity, as discussed in \autoref{sec:iter_refine_done_right} and \autoref{sec:in_depth_analysis}, respectively.

\subsubsection{Comparison with Noise Control Baselines}

To validate the superiority of our proposed noise control method, AdaTaiLr, we replaced it with other baselines, optimized their hyper-parameters, and kept all other experimental settings constant. For the similarity filtering method \cite{nguyen_improving_2023}, we used UMT \cite{li_umt_2023} as the filtering model.

As shown in \autoref{tab:noise_control_comp}, AdaTaiLr achieves state-of-the-art performance among all noise control methods with 1.2\% improvements. 
Notably, CTPR \cite{feng_ctpr_2023} significantly decreases performance, likely because its loss hypothesis (a mixture of three Gaussians) does not hold in real distributions. Loss Truncation \cite{loss_truncation} slightly decreases performance due to its main assumption that higher loss indicates larger noise, which is not supported. Evidence for these invalid hypotheses is provided in \autoref{fig:internvid_sample_loss} in the Appendix.
Additionally, we found that similarity filtering \cite{nguyen_improving_2023} can enhance performance. However, it has the drawback of filtering out 60\% of annotations, thereby reducing dataset diversity. Thus, its performance improved in the smallest dataset MSVD \cite{msvd_2011}, but decreased in the largest dataset VATEX \cite{vatex_2019}.

\subsubsection{Iterative Refinement Done Right: Insights from Ablation Studies} \label{sec:iter_refine_done_right}

\begin{figure}[t]
    \centering
    \includegraphics[width=0.66\linewidth]{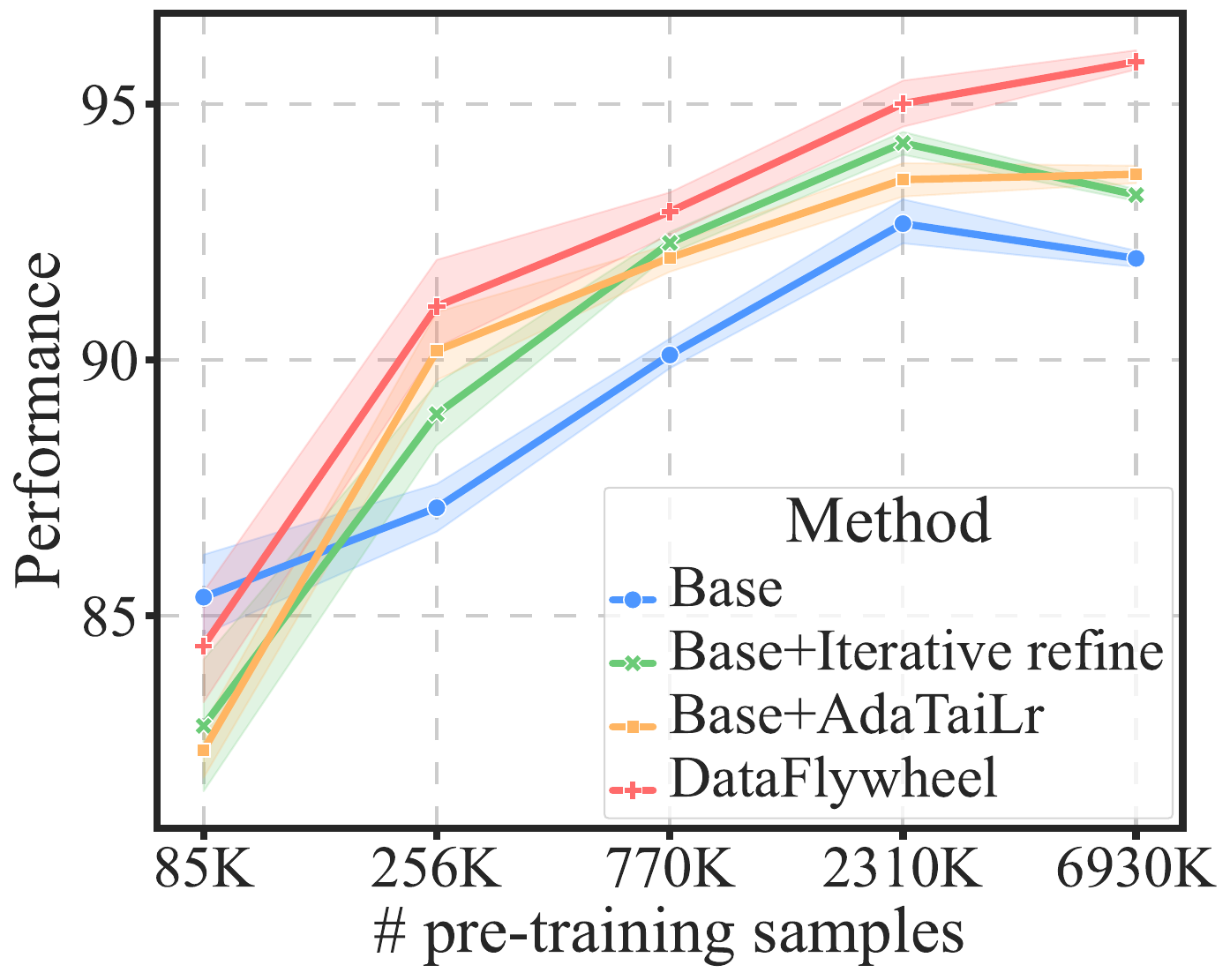}
    \caption{Ablation studies of video dataflywheel framework.}
    \label{fig:ablation_studies}
\end{figure}

We performed ablation studies to evaluate the effectiveness of the noise control method AdaTaiLr and iterative refinement. Additionally, we discovered intriguing cooperative effects between them. We followed the experimental settings described in \autoref{sec:dataflywheel_exp_settings}, except for modifying the text annotations of the refined dataset into four variants: (i) \textit{Base}: original InternVid annotations w/o AdaTaiLr during pre-training, (ii) \textit{Base+Iterative refine}: annotations without AdaTaiLr, (iii) \textit{Base+AdaTaiLr}: original InternVid annotations w/ AdaTaiLr during pre-training, and (iv) \textit{DataFlywheel}: the final refined annotations. 
The results, illustrated in \autoref{fig:ablation_studies}, yield three key observations:
\begin{itemize}
    \item \textbf{Iterative refinement alone does not scale.} Compared with \textit{Base} setting, \textit{Base+Iterative refine} can significantly improve performance. However, when we scaled the number of training data, the performance dropped significantly. So did the \textit{Base} setting. We hypothesize this is due to error accumulation from data noise.
    \item \textbf{Iterative refinement + noise control = better scalability.} When iterative refinement is combined with noise control in \textit{Base+AdaTaiLr}, performance improves as the training data scales.
    \item \textbf{VidDF performs poorly under insufficient data.} When the pre-training dataset is very small (85K), all three variants perform worse than the \textit{Base} setting. 
    For \textit{Base+AdaTaiLr}, this is likely because the convergence hypothesis does not hold during the early stages of training. As stated in \autoref{thm:aog}, some warm-up steps are required so that there exists $D>0$ under which $\left \| p_\theta - p_o \right \|_1 \le 2D$.
    For \textit{Base+Iterative refine}, this is probably due to the lack of diversity in synthetic captions when the data size is small. This will be detailed in \autoref{sec:analysis_dataflywheel}.
\end{itemize}

\subsubsection{Sensitivity Analysis}

\begin{figure}[t]
    \centering
    \includegraphics[width=0.99\linewidth]{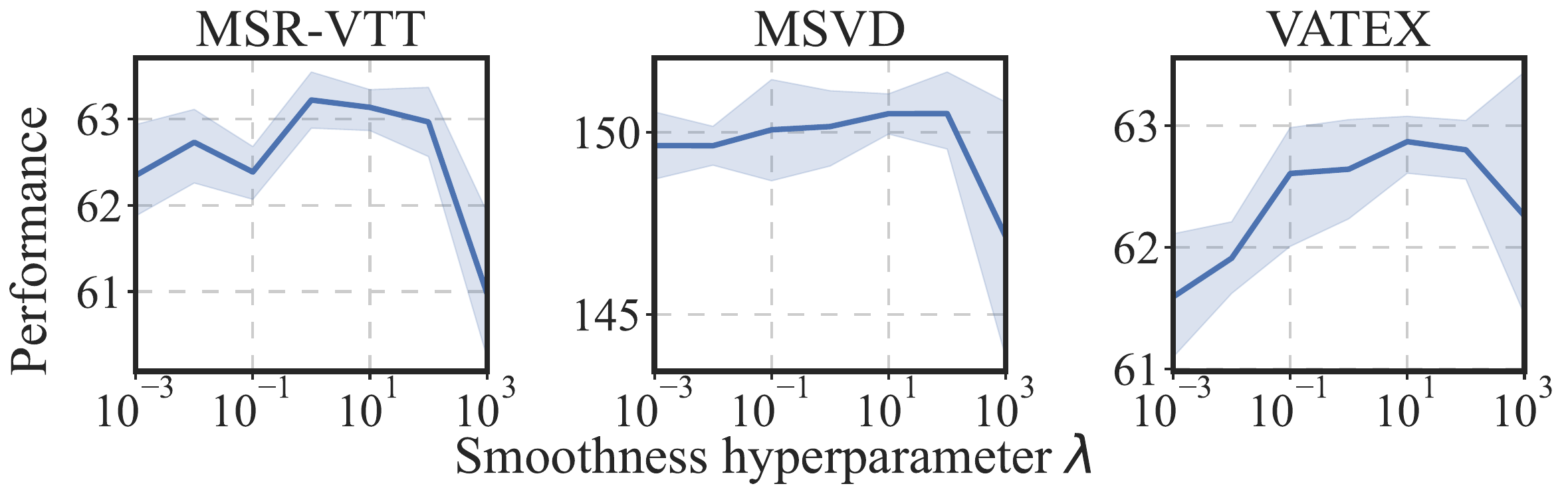}
    \caption{Sensitivity analysis of $\lambda$ controlling the smoothness of approximation in AdaTaiLr.}
    \label{fig:sensitivity_analysis_lambda}
\end{figure}

\begin{figure}[t]
    \centering
    \includegraphics[width=0.99\linewidth]{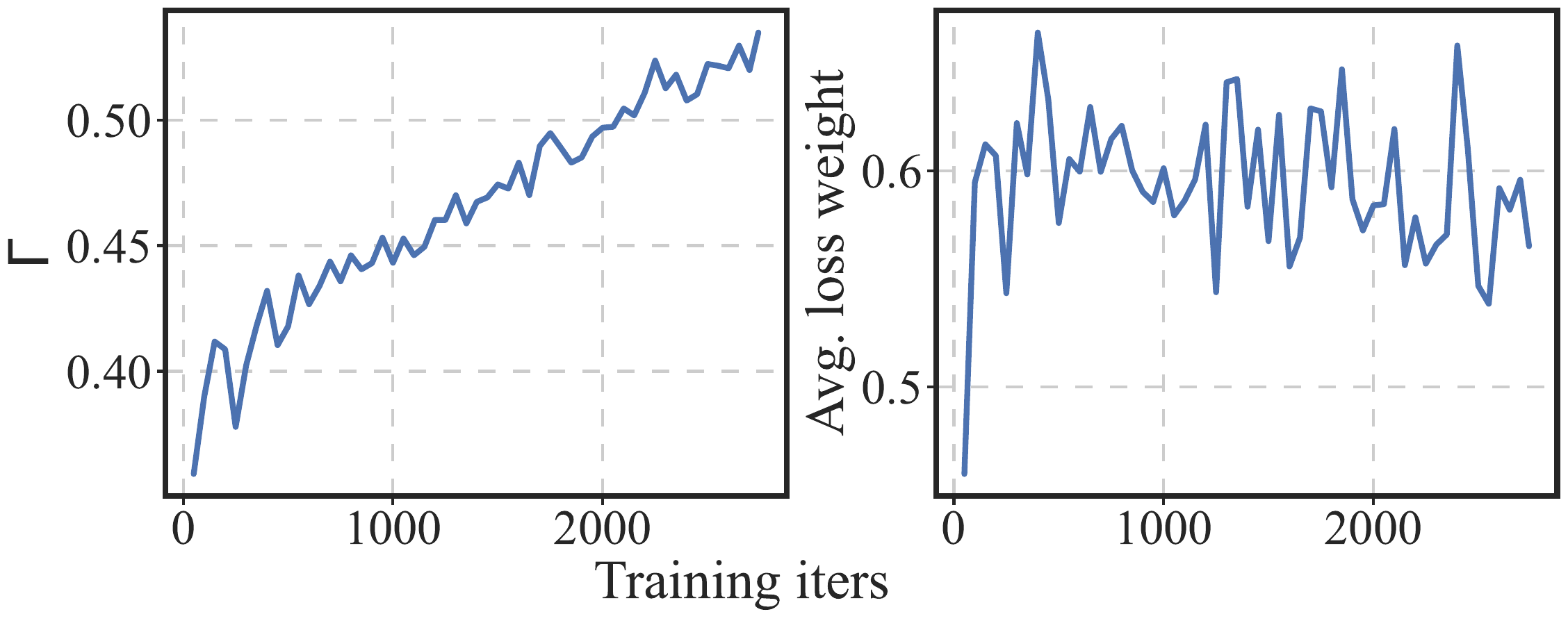}
    \caption{Sensitivity analysis of training dynamics of AdaTaiLr.}
    \label{fig:sensitivity_analysis_training}
\end{figure}

We conducted experiments to analyze the sensitivity of our methods on both hyper-parameters and training dynamics.

We first tuned the hyper-parameter $\lambda$ controlling the smoothness of approximation in \autoref{thm:aog} across a wide range, as illustrated in \autoref{fig:sensitivity_analysis_lambda}. The performance initially increases with $\lambda$, then saturates, and finally decreases sharply when $\lambda$ becomes very large. We explain such phenomenon below:
\begin{itemize}
    \item The increase-and-saturation phenomenon aligns with \autoref{thm:aog}, in which the approximation error is given by $\frac{a}{\lambda}+bD$.
    \item The sharp performance drop at high $\lambda$ values likely results from training instability, leading to a large $\left \| p_\theta - p_o \right \|_1$ and consequently a larger $D$ and approximation error.
\end{itemize}

We then analyzed the training dynamics, i.e., how trade-off factor $\Gamma$ in \autoref{eq:Gamma_appo} and LM loss weight in \autoref{eq:AdaTaiLr}
\begin{equation}
    \frac{p_\theta^{<t}(y_t)}{\Gamma + (1 - \Gamma)p_\theta^{<t}(y_t)}
\end{equation}
changes throughout training. The results, illustrated in \autoref{fig:sensitivity_analysis_training}, show the average values during training. We can gain two observations:
\begin{itemize}
    \item The trade-off factor $\Gamma$ increases during training. This aligns with our theory that $(1-\Gamma)$ controlling the bias and $\Gamma$ controlling the variance. Early in training, the model should focus on reducing bias, while later it should focus on reducing variance.
    \item Additionally, we observed that the loss weight stabilized within a certain range after the initial warm-up stage. This is expected, as the loss weight reflects the ratio of noisy data, which should remain stable throughout the training process.
\end{itemize}

\subsection{In-depth Analysis on DataFlywheel} \label{sec:in_depth_analysis}

\begin{figure}[t]
    \centering
    \includegraphics[width=0.99\linewidth]{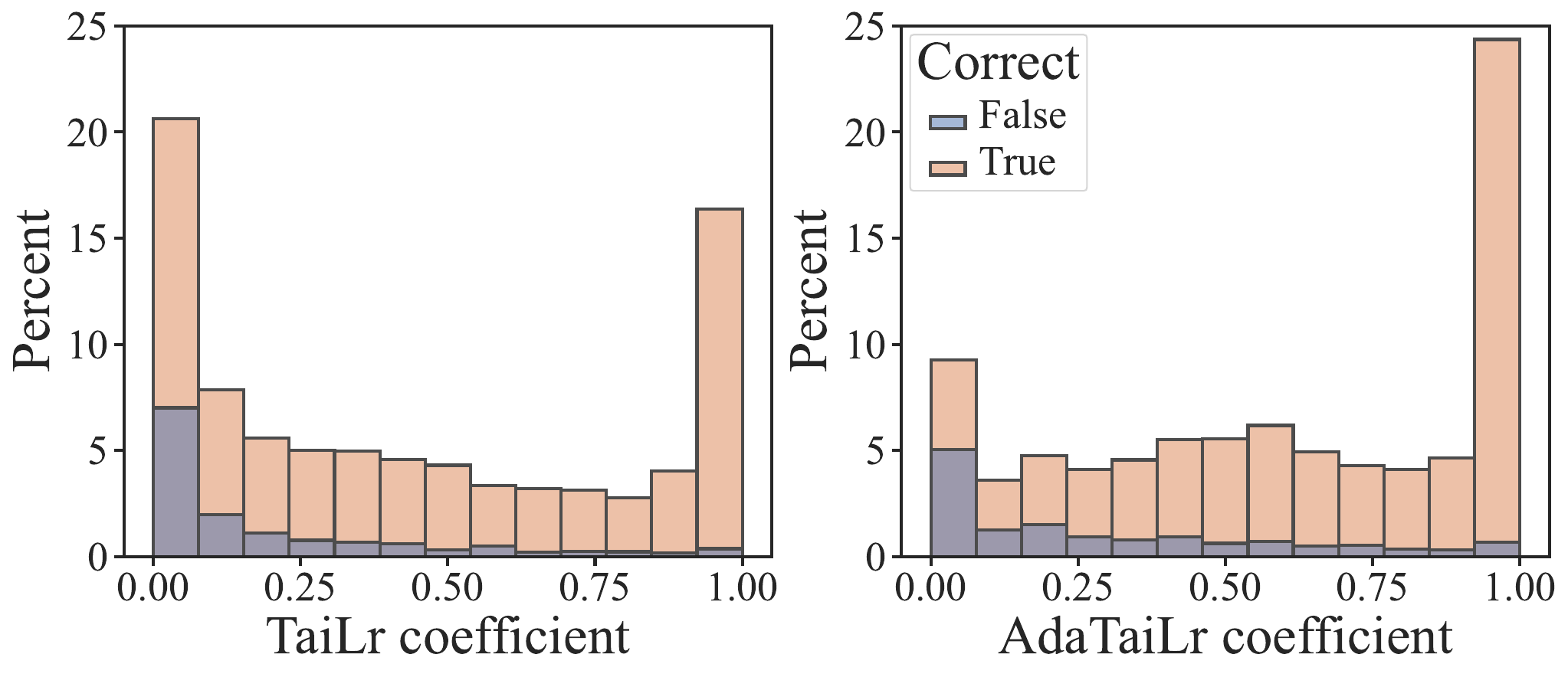}
    \caption{Comparison between distributions of the re-weight coefficient of TaiLr \cite{tailr} and AdaTaiLr Loss. AdaTaiLr can better distinguish correct text annotations.}
    \label{fig:reweight_score_comp}
\end{figure}

In this section, we delved deeper into our data flywheel framework to understand the mechanisms that make AdaTaiLr and the entire VidDF framework effective.

\subsubsection{Analysis on AdaTaiLr}

We conducted experiments to understand the working mechanism of AdaTaiLr and to determine why it outperforms TaiLr \cite{tailr} except for the theoretical enhancements.

\textbf{AdaTaiLr more effectively distinguishes correct annotations.} Given that AdaTaiLr is a loss-weighting method, we investigated how well its re-weighting coefficient corresponds with the correctness of annotations. We randomly sampled 300 video-text pairs excluded from the training set, annotated the token-level correctness of text annotations, and calculated their re-weighting coefficients for both AdaTaiLr and TaiLr \cite{tailr}. As shown in \autoref{fig:reweight_score_comp}, the re-weighting coefficient of AdaTaiLr more accurately reflects the correctness of annotations.

\begin{figure}[t]
    \centering
    \includegraphics[width=0.99\linewidth]{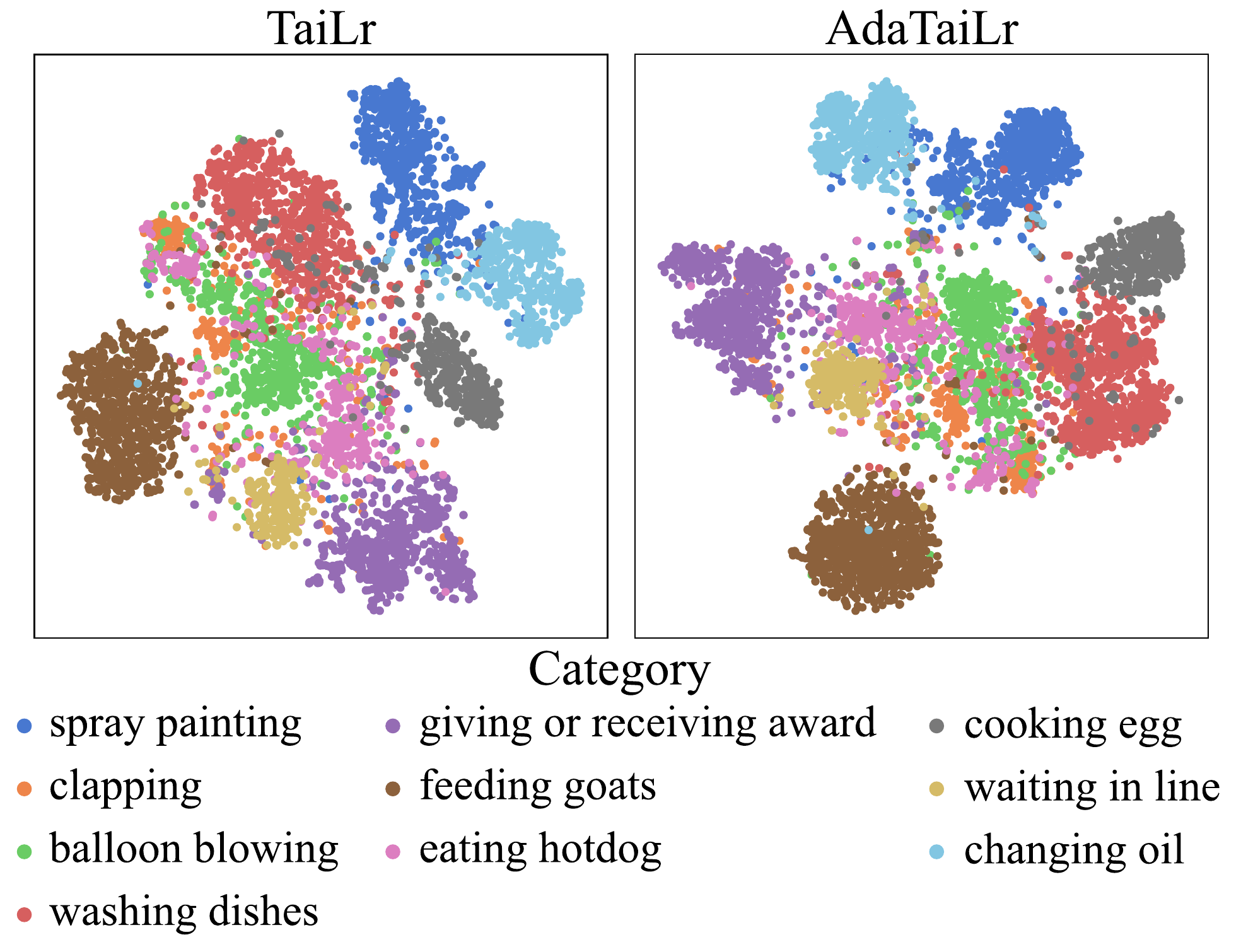}
    \caption{Visualization of video representation quality with and without AdaTaiLr as noise control.}
    \label{fig:tsne_vis}
\end{figure}

\textbf{AdaTaiLr enhances video representation learning.} To examine how AdaTaiLr influences video representation, we calculated video embeddings of all videos from 10 randomly sampled categories in the Kinetics-400 \cite{carreira_kinetics_2017} training set using the vision encoder. We sampled 8 frames per video and averaged the patch embeddings across all frames and patches to obtain the video embedding. 
The results, illustrated in \autoref{fig:tsne_vis}, show that AdaTaiLr produces more cohesive intra-class representations. For instance, the clusters for "feeding coats," "washing dishes," and "spray painting" are denser with AdaTaiLr compared to TaiLr. Additionally, AdaTaiLr yields more distinctive inter-class representations, with clearer boundaries between categories such as "spray painting" and "changing oil," or "waiting in line" and "giving or receiving award."

\subsubsection{Analysis on DataFlywheel} \label{sec:analysis_dataflywheel}

\begin{figure}[t]
    \centering
    \includegraphics[width=0.99\linewidth]{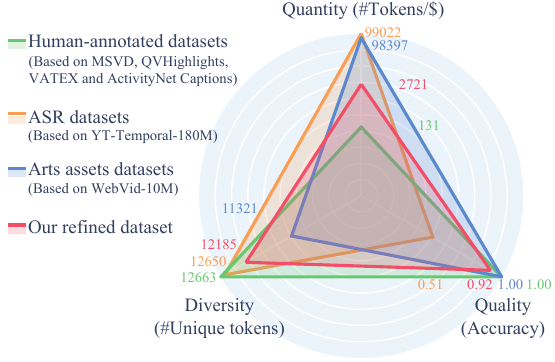}
    \caption{Data trinity comparison between existing pre-training datasets and our refined dataset.
    For datasets produced by other refinement baselines, they are close in quantity and we compare their quality and diversity in \autoref{tab:data_refine_comp} and \autoref{fig:diversity_comp}, respectively.
    }
    \label{fig:impossible_trinity}
\end{figure}

\begin{figure}[t]
    \centering
    \includegraphics[width=0.99\linewidth]{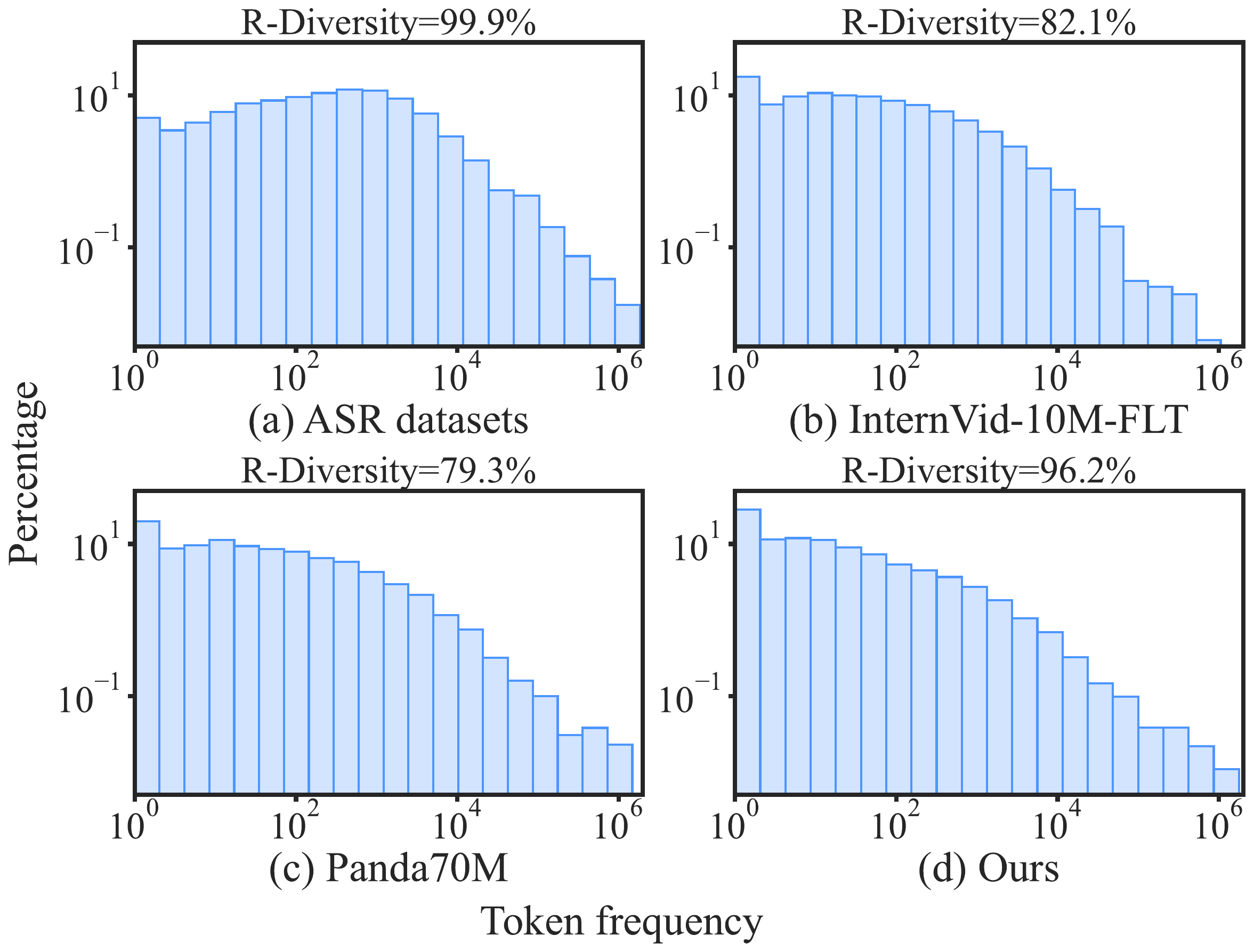}
    \caption{Comparison between dataset diversity among refined datasets.}
    \label{fig:diversity_comp}
\end{figure}

We conducted experiments to understand why the proposed VidDF framework is more effective than other data refinement methods.

\textbf{VidDF breaks the data impossible trinity.} We evaluated the data trinity of our refined dataset, as shown in \autoref{fig:impossible_trinity}. Compared to the original ASR dataset, our dataset improves the data quality dramatically with little loss of diversity and quantity. In other words, our method resolves data impossible trinity.

\textbf{DataFlywheel ensures quality with less diversity loss.} We ploted the distribution of token frequency in \autoref{fig:diversity_comp}. Our dataset has the largest diversity (96.2\% vs 82.1\%) among refined dataset baselines. This is attributed to our AdaTaiLr noise control methods, which make the most use of all annotations. In contrast, filtering out low-similarity annotations using retrieval models like InternVid \cite{wang_internvid_2023} or Panda70M \cite{chen_panda_70m_2024} significantly reduces diversity. This phenomenon is also observed in image-text pretraining \cite{nguyen_improving_2023}.

\textbf{Synthetic captions have long-tailed distribution.}We observed that all refined datasets with synthetic captions (\autoref{fig:diversity_comp} (b)-(d)) exhibit a long-tailed distribution in token frequency, unlike the uni-modal distribution of real datasets (\autoref{fig:diversity_comp} (a)). This partially explains our observation in \autoref{sec:iter_refine_done_right} that DataFlywheel performs poorly with insufficient data. Without enough data, the dataset suffers from low diversity due to the long-tailed distribution.

\subsection{Video Question Answering} \label{sec:videoqa_exps}

\begin{table*}[t]

\caption{
Performance comparison in video question answering benchmark.
}

\centering
\label{tab:videoqa_results}

\begin{tabular}{ccclcclcclcclcc}
\hline
 &
   &
   &
   &
  \multicolumn{2}{c}{MSVD} &
   &
  \multicolumn{2}{c}{MSRVTT} &
   &
  \multicolumn{2}{c}{ActivityNet} &
   &
  \multicolumn{2}{c}{TGIF} \\ \cline{5-6} \cline{8-9} \cline{11-12} \cline{14-15} 
\multirow{-2}{*}{Method} &
  \multirow{-2}{*}{\begin{tabular}[c]{@{}c@{}}Vision\\ Encoder\end{tabular}} &
  \multirow{-2}{*}{\begin{tabular}[c]{@{}c@{}}LLM\\ Size\end{tabular}} &
   &
  Acc. &
  Sco. &
   &
  Acc. &
  Sco. &
   &
  Acc. &
  Sco. &
   &
  Acc. &
  Sco. \\ \hline
FrozenBiLM \cite{yang_frozenbilm_2022} &
  ViT-L &
  1.3B &
   &
  33.8 &
  - &
   &
  16.7 &
  - &
   &
  25.9 &
  - &
   &
  41.9 &
  - \\
Video-LLaMA  \cite{video_llama_2023} &
  CLIP-G &
  7B &
   &
  51.6 &
  2.5 &
   &
  29.6 &
  1.8 &
   &
  12.4 &
  1.1 &
   &
  - &
  - \\
LLaMA-Adapter \cite{zhang_llama_adapter_2023} &
  ViT-B &
  7B &
   &
  54.9 &
  3.1 &
   &
  43.8 &
  2.7 &
   &
  34.2 &
  2.7 &
   &
  - &
  - \\
Video-ChatGPT \cite{video_chatgpt_2023} &
  ViT-L &
  7B &
   &
  64.9 &
  3.3 &
   &
  49.3 &
  2.8 &
   &
  35.2 &
  2.7 &
   &
  51.4 &
  3.0 \\
Video-LLaVA \cite{video_llava} &
  ViT-L &
  7B &
   &
  70.7 &
  3.9 &
   &
  59.2 &
  3.5 &
   &
  45.3 &
  3.3 &
   &
  70.0 &
  4.0 \\
Chat-UniVi \cite{chat_univi_2023} &
  ViT-L &
  7B &
   &
  65.0 &
  3.6 &
   &
  54.6 &
  3.1 &
   &
  45.8 &
  3.2 &
   &
  60.3 &
  3.4 \\
MovieChat \cite{moviechat} &
  CLIP-G &
  7B &
   &
  75.2 &
  3.8 &
   &
  52.7 &
  2.6 &
   &
  45.7 &
  3.4 &
   &
  - &
  - \\
VideoChat \cite{videochat_2023} &
  CLIP-G &
  7B &
   &
  56.3 &
  2.8 &
   &
  45.0 &
  2.5 &
   &
  26.5 &
  2.2 &
   &
  34.4 &
  2.3 \\
VideoChat2 \cite{videochat2} &
  UMT-L &
  7B &
   &
  70.0 &
  3.9 &
   &
  54.1 &
  3.3 &
   &
  49.1 &
  3.3 &
   &
  - &
  - \\
Vista-LLaMA \cite{ma_vista_llama_2023} &
  CLIP-G &
  7B &
   &
  65.3 &
  3.6 &
   &
  60.5 &
  3.3 &
   &
  48.3 &
  3.3 &
   &
  - &
  - \\
LLaMA-VID \cite{llama_vid} &
  CLIP-G &
  13B &
   &
  70.0 &
  3.7 &
   &
  58.9 &
  3.3 &
   &
  47.5 &
  3.3 &
   &
  - &
  - \\
ST-LLM \cite{st_llm_2024} &
  BLIP2 &
  7B &
   &
  74.6 &
  3.9 &
   &
  63.2 &
  3.4 &
   &
  50.9 &
  3.3 &
   &
  - &
  - \\
LLaVA-NEXT 7B \cite{llava_next} &
  ViT-L &
  7B &
   &
  78.8 &
  4.1 &
   &
  63.7 &
  3.5 &
   &
  54.3 &
  3.4 &
   &
  73.0 &
  4.0 \\
{\color[HTML]{939393} LLaVA-NEXT 13B \cite{llava_next}} &
  {\color[HTML]{939393} ViT-L} &
  {\color[HTML]{939393} 13B} &
   &
  {\color[HTML]{939393} 77.4} &
  {\color[HTML]{939393} 4.1} &
   &
  {\color[HTML]{939393} 62.6} &
  {\color[HTML]{939393} 3.4} &
   &
  {\color[HTML]{939393} 57.1} &
  {\color[HTML]{939393} 3.5} &
   &
  {\color[HTML]{939393} 78.0} &
  {\color[HTML]{939393} 4.0} \\
{\color[HTML]{939393} LLaVA-NEXT 34B \cite{llava_next}} &
  {\color[HTML]{939393} ViT-L} &
  {\color[HTML]{939393} 34B} &
   &
  {\color[HTML]{939393} 79.6} &
  {\color[HTML]{939393} 4.1} &
   &
  {\color[HTML]{939393} 62.4} &
  {\color[HTML]{939393} 3.5} &
   &
  {\color[HTML]{939393} 58.4} &
  {\color[HTML]{939393} 3.5} &
   &
  {\color[HTML]{939393} 79.1} &
  {\color[HTML]{939393} 4.2} \\
PLLaVA 7B \cite{xu_pllava_2024} &
  ViT-L &
  7B &
   &
  76.6 &
  4.1 &
   &
  62.0 &
  3.5 &
   &
  56.3 &
  3.5 &
   &
  77.5 &
  4.1 \\
{\color[HTML]{939393} PLLaVA 13B \cite{xu_pllava_2024}} &
  {\color[HTML]{939393} ViT-L} &
  {\color[HTML]{939393} 13B} &
   &
  {\color[HTML]{939393} 75.7} &
  {\color[HTML]{939393} 4.1} &
   &
  {\color[HTML]{939393} 63.2} &
  {\color[HTML]{939393} 3.6} &
   &
  {\color[HTML]{939393} 56.3} &
  {\color[HTML]{939393} 3.6} &
   &
  {\color[HTML]{939393} 77.8} &
  {\color[HTML]{939393} 4.2} \\
{\color[HTML]{939393} PLLaVA 34B \cite{xu_pllava_2024}} &
  {\color[HTML]{939393} ViT-L} &
  {\color[HTML]{939393} 34B} &
   &
  {\color[HTML]{939393} 79.9} &
  {\color[HTML]{939393} 4.2} &
   &
  {\color[HTML]{939393} 68.7} &
  {\color[HTML]{939393} 3.8} &
   &
  {\color[HTML]{939393} 60.9} &
  {\color[HTML]{939393} 3.7} &
   &
  {\color[HTML]{939393} 80.6} &
  {\color[HTML]{939393} 4.3} \\ \hline
\textbf{Ours} &
  ViT-L &
  7B &
   &
  \textbf{79.0} &
  \textbf{4.1} &
   &
  \textbf{64.5} &
  \textbf{3.6} &
   &
  \textbf{57.5} &
  \textbf{3.6} &
   &
  \textbf{78.4} &
  \textbf{4.2} \\ \hline
\end{tabular}

\end{table*}

\begin{table*}[t]

\caption{
Performance comparison in text-to-video retrieval benchmark.
}

\centering
\label{tab:videoret_results}

\begin{tabular}{ccccccccccccccccc}
\hline
                                                            &                             & \multicolumn{3}{c}{MSRVTT}                                                              &                         & \multicolumn{3}{c}{DiDeMo}                                                              &                         & \multicolumn{3}{c}{ActivityNet}                                                         &                         & \multicolumn{3}{c}{MSVD}                                                                \\
\multirow{-2}{*}{Method}                                    & \multirow{-2}{*}{\#Pairs}   & R@1                         & R@5                         & R@10                        &                         & R@1                         & R@5                         & R@10                        &                         & R@1                         & R@5                         & R@10                        &                         & R@1                         & R@5                         & R@10                        \\ \hline
OmniVL \cite{wang_omnivl_2022}                              & 17M                         & 47.8                        & 74.2                        & 83.8                        &                         & 52.4                        & 79.5                        & 85.4                        &                         & -                           & -                           & -                           &                         & -                           & -                           & -                           \\
VINDLU \cite{cheng_vindlu_2023}                             & 25M                         & 48.8                        & 72.4                        & 82.2                        &                         & 59.8                        & 86.6                        & 91.5                        &                         & 55.9                        & 82.3                        & 90.9                        &                         & -                           & -                           & -                           \\
RTQ \cite{wang_rtq_2023}                                    & 129M                        & 53.4                        & 76.1                        & 84.4                        &                         & 57.6                        & 84.1                        & 89.8                        &                         & 53.5                        & 81.4                        & 91.9                        &                         & -                           & -                           & -                           \\
{\color[HTML]{939393} PIDRo \cite{guan_pidro_2023}}         & {\color[HTML]{939393} 400M} & {\color[HTML]{939393} 50.2} & {\color[HTML]{939393} 77.0} & {\color[HTML]{939393} 85.4} & {\color[HTML]{939393} } & {\color[HTML]{939393} 48.6} & {\color[HTML]{939393} 75.9} & {\color[HTML]{939393} 84.4} & {\color[HTML]{939393} } & {\color[HTML]{939393} 44.9} & {\color[HTML]{939393} 74.5} & {\color[HTML]{939393} 86.1} & {\color[HTML]{939393} } & {\color[HTML]{939393} 47.5} & {\color[HTML]{939393} 77.5} & {\color[HTML]{939393} 86.0} \\
{\color[HTML]{939393} Intern Video \cite{internvideo_2022}} & {\color[HTML]{939393} 646M} & {\color[HTML]{939393} 55.2} & {\color[HTML]{939393} 79.6} & {\color[HTML]{939393} 87.5} & {\color[HTML]{939393} } & {\color[HTML]{939393} 57.9} & {\color[HTML]{939393} 82.4} & {\color[HTML]{939393} 88.9} & {\color[HTML]{939393} } & {\color[HTML]{939393} 62.2} & {\color[HTML]{939393} 85.9} & {\color[HTML]{939393} 93.2} & {\color[HTML]{939393} } & {\color[HTML]{939393} 58.4} & {\color[HTML]{939393} 84.5} & {\color[HTML]{939393} 90.4} \\
{\color[HTML]{939393} CLIP-ViP \cite{xue_clip-vip_2023}}    & {\color[HTML]{939393} 500M} & {\color[HTML]{939393} 54.2} & {\color[HTML]{939393} 77.2} & {\color[HTML]{939393} 84.8} & {\color[HTML]{939393} } & {\color[HTML]{939393} 50.5} & {\color[HTML]{939393} 78.4} & {\color[HTML]{939393} 87.1} & {\color[HTML]{939393} } & {\color[HTML]{939393} 53.4} & {\color[HTML]{939393} 81.4} & {\color[HTML]{939393} 90.0} & {\color[HTML]{939393} } & {\color[HTML]{939393} -}    & {\color[HTML]{939393} -}    & {\color[HTML]{939393} -}    \\
UMT-L \cite{li_umt_2023}                                    & 5M                          & 53.3                        & 76.6                        & 83.9                        &                         & 59.7                        & 84.9                        & 90.8                        &                         & 58.1                        & 85.5                        & 92.9                        &                         & 53.7                        & 80.5                        & 86.8                        \\
ViCLIP+InternVid \cite{wang_internvid_2023}                 & 10M                         & 55.0                        & -                           & -                           &                         & 51.7                        & -                           & -                           &                         & 50.4                        & -                           & -                           &                         & 53.9                        & -                           & -                           \\
UMT+Panda70M \cite{chen_panda_70m_2024}                     & 5M                          & \textbf{58.4}               & 80.9                        & 86.9                        &                         & 60.6                        & 86.0                        & 92.4                        &                         & -                           & -                           & -                           &                         & 57.5                        & 83.6                        & 89.5                        \\ \hline
\textbf{Ours}                                               & 5M                          & 56.1                        & \textbf{81.6}               & \textbf{87.0}               &                         & \textbf{66.6}               & \textbf{86.7}               & \textbf{93.1}               &                         & \textbf{64.9}               & \textbf{88.0}               & \textbf{94.2}               &                         & \textbf{59.1}               & \textbf{84.5}               & \textbf{89.5}               \\ \hline
\end{tabular}

\end{table*}

In this section, we integrated our refined dataset with models in video question answering to validate improvements.

\subsubsection{Experimental Settings}

\textbf{Models and implementation details}
In this section, we adopted PLLaVA \cite{xu_pllava_2024} as the VideoLLM, and followed the original training setting. Based on LLaVA-NEXT \cite{llava_next}, PLLAVA treats videos as multi-images arranged in temporal order. To reduce the number of visual tokens, it performs temporal pooling on the temporal dimension with stride 2.

\noindent\textbf{Training datasets}. 
For image data, we leverage LLaVA-Pretrain-558K \cite{liu_llava_2023} for pre-training and a 745K dataset similar to LLaVA-SFT-760K \cite{llava_next} for SFT. The LLaVA-SFT-760K is not publicly available and contains a 15K private dataset, thus we forge a 745K dataset based on the paper. 

\noindent\textbf{Evaluation datasets and metrics.} 
We incorporated various benchmarks on open-ended Video Question Answer (VideoQA) including MSVD-QA \cite{msvd_2011}, MSRVTT-QA \cite{msrvtt_2016}, ActivityNet-QA \cite{anet_caption_2017}, and TGIF-QA \cite{li_tgifqa_2016}. Considering that ground-truth answers in these benchmarks are single-word, we followed Maaz et al., \cite{video_chatgpt_2023} to prompt GPT-3.5 for evaluating the accuracy (Acc., with answers true/false) and quality (Sco., ranging from 0 to 5) of the models’ responses. We followed their original paper except we adopted \texttt{GPT-3.5-turbo-0125} for evaluation to align with recent works \cite{llava_next, xu_pllava_2024}. The GPT-3.5 version in their original version is deprecated.


\subsubsection{Performance Comparison}

Due to resource constraints, we trained only the 7B version of our model. The SFT results are presented in \autoref{tab:videoqa_results}. Our model achieves state-of-the-art performance, improving results by up to 2.1\% compared to existing baselines. 
Notably, it outperforms the 13B versions of LLaMA-VID \cite{llama_vid}, LLaVA-NEXT \cite{llava_next}, and PLLaVA \cite{xu_pllava_2024}, demonstrating the effectiveness of our refined dataset.


\subsection{Text-to-video Retrieval} \label{sec:videoret_exps}

In this section, we integrated our refined dataset with existing models to validate improvements in text-video retrieval.

\subsubsection{Experimental Settings}

\textbf{Model, training datasets, and implementation details}. 
We used the UnMasked Teacher (UMT) \cite{li_umt_2023} as the base model to evaluate the performance of text-to-video retrieval. We selected UMT-L to be in accordance with most of the baselines.
For a fair comparison, we randomly sampled 5M video-text pairs from our dataset as the pretraining data, since previous data refinement methods Panda70M \cite{chen_panda_70m_2024} and InternVid \cite{wang_internvid_2023} samples 5M and 10M, respectively.
Except for the training data, we followed the implementation details exactly as UMT stage-2 pertaining \cite{li_umt_2023}.

\noindent\textbf{Evaluation datasets and metrics.} 
We tested fine-tuned retrieval on four benchmarks: MSR-VTT \cite{msrvtt_2016}, DiDeMo \cite{hendricks_didemo_2017}, ActivityNet-Captions \cite{anet_caption_2017}, and MSVD \cite{msvd_2011}. We followed the common evaluation protocol. 
Specifically, For MSRVTT we evaluated on 1K testing split, which is not the same as the testing videos for captioning in \autoref{sec:exp_dataflywheel}. 
For DiDeMo and ActivityNet-Captions, they contain videos with dense captions. As in the previous standard \cite{li_umt_2023}, we evaluated paragraph-to-video retrieval by concatenating all descriptions of one video into a single query. We reported results on the 1K testing set. 
For MSVD, we reported results on the 670 testing videos. 
For evaluation metrics, we employed the standard metric and reported R@1, R@5, and R@10 accuracy.

\subsubsection{Performance Comparison}

As depicted in \autoref{tab:videoret_results}, our dataset significantly enhances text-video retrieval tasks. 
Compared to the original training set of UMT-L \cite{li_umt_2023}, which primarily consists of the art assets dataset WebVid \cite{bain_webvid_2021}, our dataset improves performance across all benchmarks by 5\%-6\% on average (Avg. of R@1, R@5, and R@10). 
Additionally, we outperform many existing state-of-the-art methods \cite{guan_pidro_2023, internvideo_2022, xue_clip_vip_2023} that were pre-trained with significantly more vision-text data pairs.
When compared to state-of-the-art dataset refinement methods, our approach consistently outperforms the 5M subset of Panda70M \cite{chen_panda_70m_2024} by 1\%-3\% on the DiDeMo and MSVD datasets. However, in the R@1 metric for the MSR-VTT dataset, our model is 3\% lower than Panda70M. It is worth noting that the synthetic annotations of Panda70M are filtered through a model trained with 100K human annotations, where annotators selected the best synthetic captions from eight options in the raw Panda70M dataset. These results suggest that our algorithm-based noise control may benefit from the human annotations in certain domains.

\section{Conclusion}

This study presents a quantitative analysis that underscores the "impossible trinity" challenge inherent in video-language pre-training datasets, among data quantity, diversity, and quality. Our findings provide valuable insights for the future curation, evaluation, and enhancement of such datasets.
In response to this challenge, we introduce the Video DataFlywheel framework, an innovative system that iteratively refines text annotations derived from ASR datasets. To effectively manage noise during this refinement process, we propose AdaTaiLr, a method necessitating fewer assumptions on noise distribution, thus proving particularly efficacious in larger datasets.
Comprehensive experiments validated the effectiveness of the VidDF framework, demonstrating its ability to enhance data quality with minimal loss of diversity. Besides, the VidDF framework has better scalability than existing data refinement methods. Furthermore, our refined datasets significantly improved performance in various video-language understanding tasks, including video question answering and video-text retrieval.

In the future, we plan to: 1) Enhance the framework's autonomy to actively select videos with potentially superior refinement results or unknown knowledge. 2) Develop new noise control methods that can better integrate human annotations for supplementation. 3) Integrate additional quality evaluation methods, such as aesthetics and annotation detailedness, to broaden the dataset's applicability across fields like video generation.

\section*{Acknowledgment}


\ifCLASSOPTIONcaptionsoff
  \newpage
\fi





\bibliographystyle{IEEEtran}
\bibliography{IEEEabrv,Bibliography}

%


\begin{IEEEbiography}[{\includegraphics[width=1in,height=1.25in,clip,keepaspectratio]{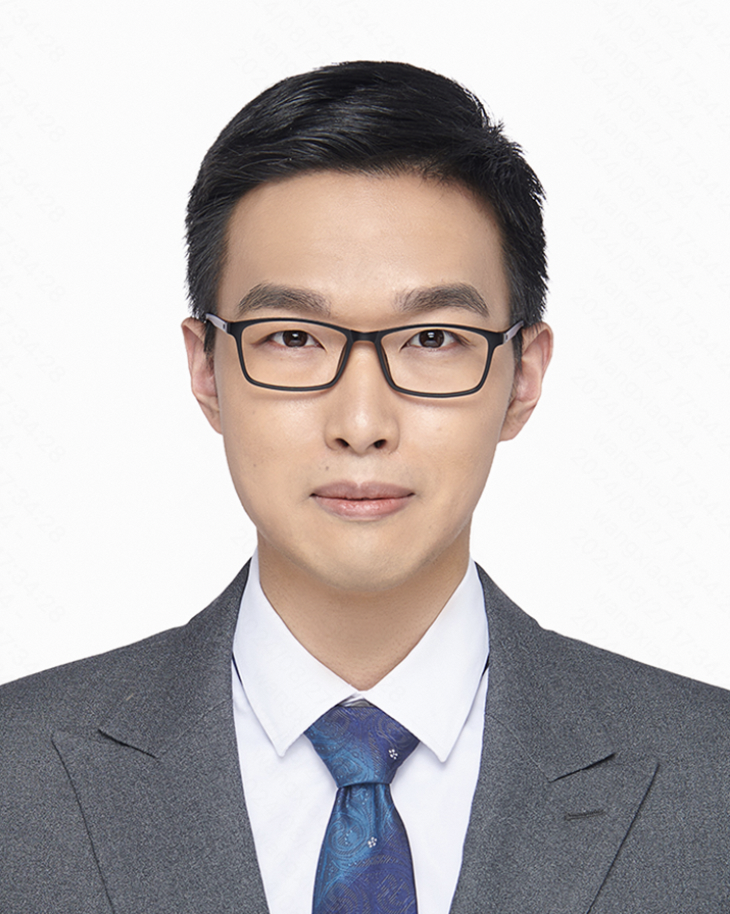}}]{Xiao Wang}
is currently working toward a PhD degree with the School of Computer Science and Technology, Harbin Institute of Technology, Shenzhen, China. He received the BS and MS degrees in Physics and Computer Science respectively from Shandong University, China. His research interests focus on video understanding based on multi-modal language models.
\end{IEEEbiography}

\begin{IEEEbiography}
[{\includegraphics[width=1in,height=1.25in,clip,keepaspectratio]{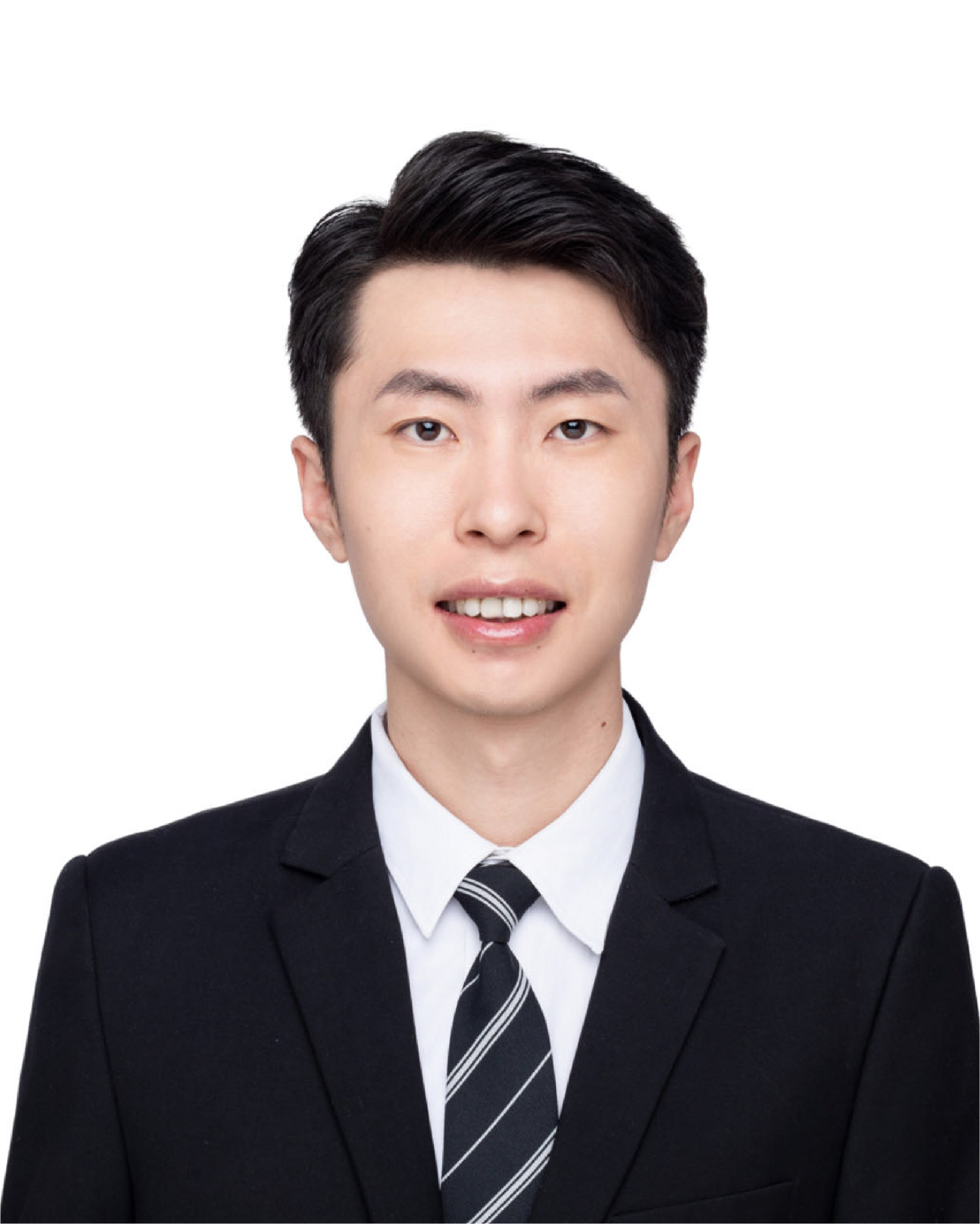}}]{Jianlong Wu}
(Member, IEEE) received the B.Eng. degree from the Huazhong University of Science and Technology, China, in 2014, and the Ph.D. degree from Peking University, China, in 2019. He is currently an Associate Professor with the Harbin Institute of Technology (Shenzhen), China. He was an Assistant Professor at the Shandong University from 2019 to 2022. He has published over 40 papers in top journals and conferences, such as IEEE TPAMI, ICML, NeurIPS, and ICCV. His research interests include computer vision and multi-modal learning. He received many awards, such as the Outstanding Reviewer of ICML 2020 and the Best Student Paper of SIGIR 2021. He serves as the Area Chair for NeurIPS and ACM MM, and reviewer for many top journals and conferences, such as IEEE TPAMI and International Journal of Computer Vision.
\end{IEEEbiography}

\begin{IEEEbiographynophoto}{Zijia Lin}
is now a technical leader in Kuaishou. He received his BS and PhD degrees from the Department of Computer Science and Technology, Tsinghua University, China. He worked with Microsoft Research Asia before joining Kuaishou. His research areas include natural language processing, information retrieval, and computer vision.
\end{IEEEbiographynophoto}

\begin{IEEEbiography}
[{\includegraphics[width=1in,height=1.25in,clip,keepaspectratio]{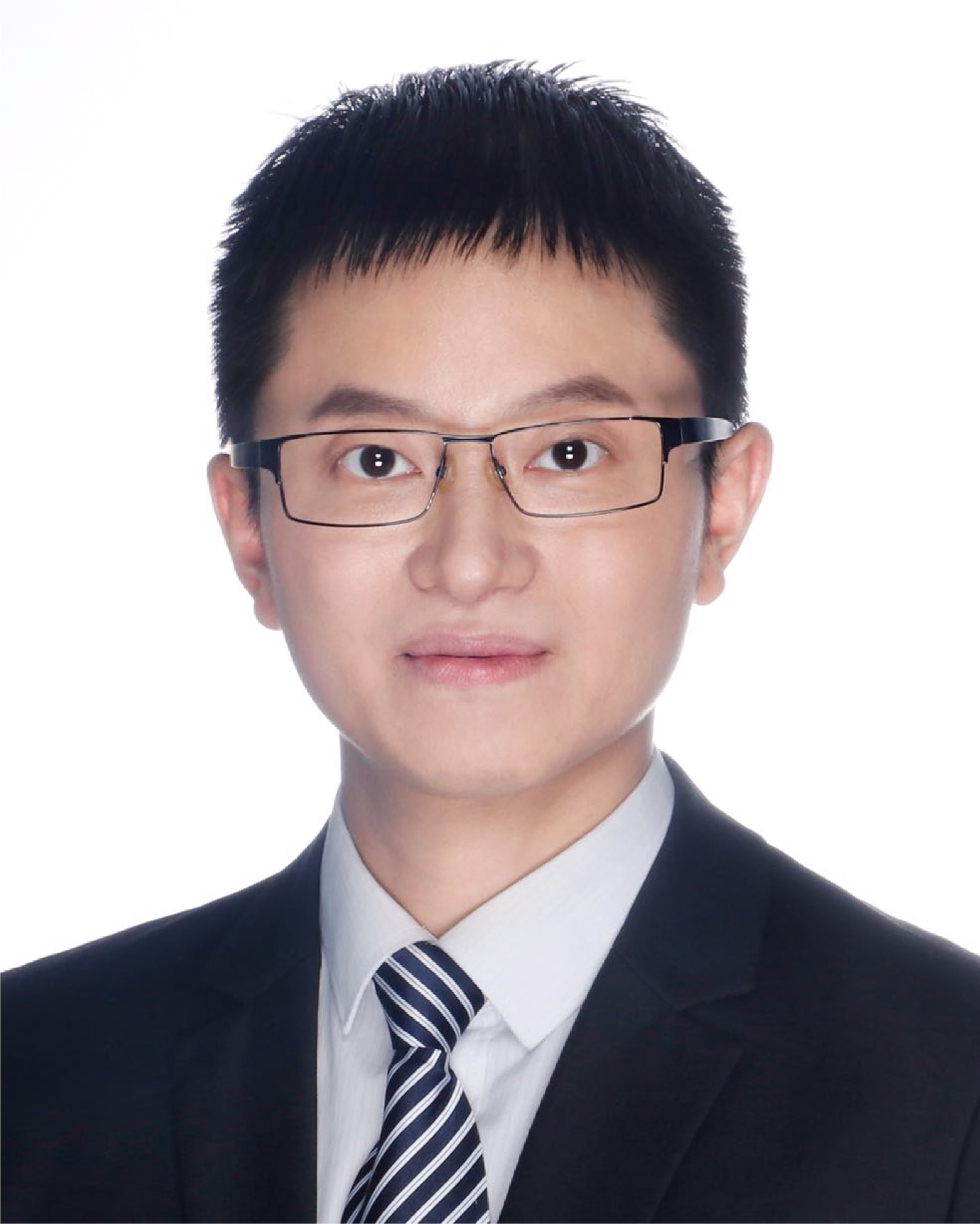}}]{Fuzheng Zhang}
is now a director in Kuaishou. His research interests include NLP, knowledge graphs, information retrieval, and recommender systems. He obtained his Ph.D. degree in computer science, and is supervised jointly by the University of Science and Technology of China and Microsoft Research Asia. Dr. Zhang has published many top-tier international conference papers and journal articles in his research area, such as KDD, WWW, AAAI, and IJCAI. He has received the best paper award in ICDM2013 and CIKM2020. Dr. Zhang is also active in academic activities. For example, he served as the industry chair in ASONAM2018 and has long served as the reviewer on top-tier international conferences and journals, such as KDD, WWW, and TKDE.
\end{IEEEbiography}

\begin{IEEEbiographynophoto}{Di Zhang}
Di Zhang is now a vice president in Kuaishou. He received his BS and MS degrees from the Department of Computer Science and Technology, Shanghai Jiao Tong University, China. He worked with Alibaba before joining Kuaishou. His research area focuses on large-scale AI infrastructure.
\end{IEEEbiographynophoto}

\begin{IEEEbiography}
[{\includegraphics[width=1in,height=1.25in,clip,keepaspectratio]{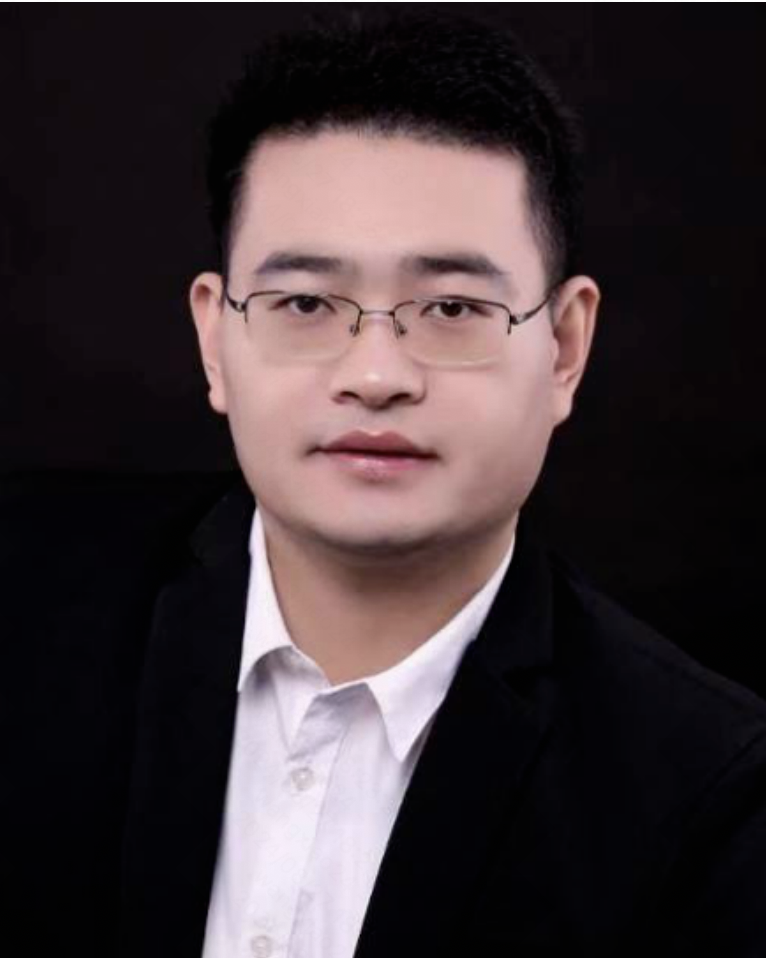}}]{Liqiang Nie}
(Senior Member, IEEE) is currently the dean of the School of Computer Science and Technology, Harbin Institute of Technology (Shenzhen). He is a fellow of IAPR and AAIA. He received his B.Eng. and Ph.D. degrees from Xi’an Jiaotong University and the National University of Singapore, respectively. His research interests lie primarily in multimedia content analysis and information retrieval. He is an AE of IEEE TKDE, IEEE TMM, IEEE TCSVT, ACM ToMM, and Information Science. Meanwhile, he is the regular AC or SPC of ACM MM, NeurIPS, IJCAI, and AAAI. He has received many awards, like the ACM MM and SIGIR Best Paper Honorable Mention in 2019, SIGMM Rising Star in 2020, SIGIR Best Student Paper in 2021, and ACM MM Best Paper Award in 2022.
\end{IEEEbiography}




\vfill



%

\appendices

\section{The Impossible Data Trinity} \label{sec:sup_imp_trinity}

\subsection{Quantity}

\begin{table}[t]

\caption{
Cost of human-annotated datasets.
}

\label{tab:anno_cost_human}
\resizebox{\columnwidth}{!}{

\begin{tabular}{ccccc}
\hline
\textbf{Dataset name}                                                                   & \textbf{\#Captions} & \textbf{\#Tokens} & \textbf{\$/Caption} & \textbf{\#Tokens/\$} \\ \hline
MSVD \cite{msvd_2011}                 & 81K  & 0.70M & 0.05 & 173 \\
VATEX \cite{vatex_2019}               & 350K & 6.38M & 0.12 & 152 \\
\begin{tabular}[c]{@{}c@{}}ActivityNet\\ Captions \cite{anet_caption_2017}\end{tabular} & 72K                & 1.24M            & 0.12                & 144               \\
QVHighlights \cite{qvhighlights_2021} & 10K  & 0.14M & 0.25 & 56   \\ \hline
\end{tabular}

}
\end{table}

\begin{table}[t]

\caption{
Cost of the art asset and ASR datasets. 36C 96G$\times$ 2 means 2 cloud instances each with 36 CPU and 96GB memory. We only download a subset of YT-Temporal-180M \cite{zellers_yt-tmp-180m_2021} for cost estimation.
}

\label{tab:anno_cost_crawl}
\centering

\begin{tabular}{ccc}
\hline
\textbf{Dataset name}    & YT-Temporal \cite{zellers_yt-tmp-180m_2021} & WebVid \cite{bain_webvid_2021} \\ \hline
\textbf{\#Captions}             & 6,665,285   & 9,895,441   \\
\textbf{\#Tokens}               & 213,289,120 & 227,814,515 \\
\textbf{Cloud instances} & 36C 96G $\times$ 2                          & 8C 16G $\times$ 32             \\
\textbf{\$/(instance*h)}       & 1.944       & 0.272       \\
\textbf{Download duration (h)} & 554         & 266         \\
\textbf{\#Tokens/\$}              & 99022       & 98397       \\ \hline
\end{tabular}

\end{table}

\textbf{Definition.}
Quantity refers to the number of text annotations we can collect under a certain budget:
\begin{equation}
    \textmd{Quantity}=
    \frac{
        \left | \textmd{Tokens in text annotations} \right | 
    }
    {
        \textmd{Dataset collection cost (in \$)}
    }.
\end{equation}
We use the number of tokens instead of sentences to measure the number of text annotations, since the length of each sentence varies between datasets. We use the same tokenizer as Vicuna 1.5 \cite{vicuna_llm}.
We do not compare the number of annotations among datasets directly, since some datasets are collected with similar methods but varying budgets (e.g. YT-Temporal 2B \cite{zellers_yt-tmp-2b_2022} and HD-VILA 100M \cite{xue_hd-vila_2022}).

\textbf{Calculation.}
For human-annotated datasets, we list datasets that the collection cost is revealed in their paper in \autoref{tab:anno_cost_human}. 
For arts assert and ASR datasets, we estimate their collection cost through experiments. Specifically, we download subsets of YT-Temporal-180M \cite{zellers_yt-tmp-180m_2021} and WebVid \cite{bain_webvid_2021} datasets, and estimate the collection cost through the resources consumed. The collection is composed of data crawling and downloading costs. The downloading cost is the majority since it is a compute-heavy task. The results are presented in \autoref{tab:anno_cost_crawl}. The two datasets use different cloud instances because we chose them to optimize the utilization rate. The ASR datasets require video clipping after downloading, requiring more CPU cores per instance. The price of could instances is calculated based on Amazon Web Services Pricing Calculator\footnote{\url{https://calculator.aws/\#/?nc2=h_ql_pr_calc}}. Note that the whole downloading process is completed in clusters of our own company. We chose Amazon Elastic Compute Cloud instances that share similar abilities with ours (c6g.2xlarge for WebVid and c5n.9xlarge for YT-Temporal).

\subsection{Diversity}

\begin{figure}[t]
    \centering
    \includegraphics[width=0.66\linewidth]{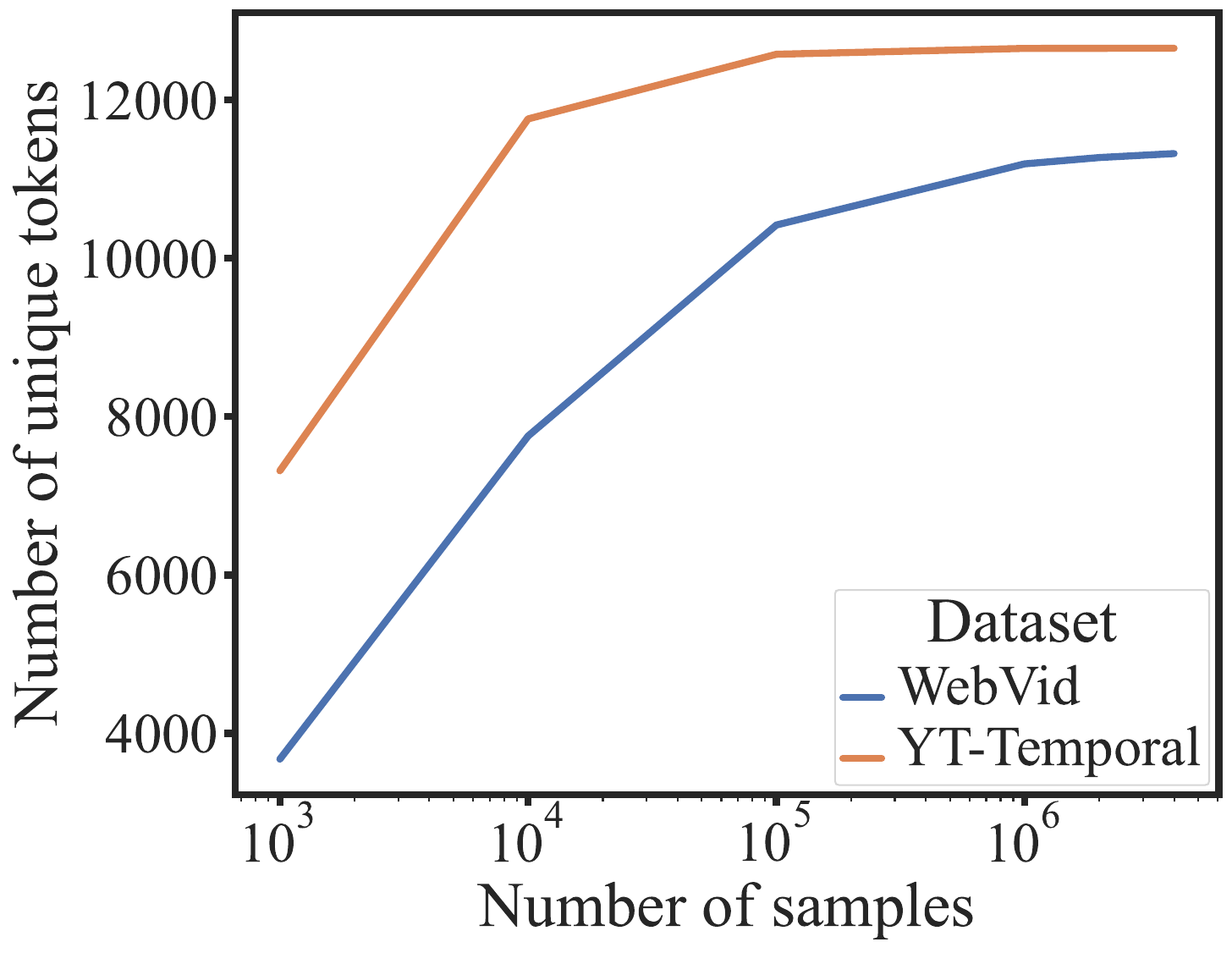}
    \caption{Number of unique tokens saturates with the increase of sampled text annotations.}
    \label{fig:diversity_scaling}
\end{figure}

\begin{table}[t]

\caption{
Diversity of the art asset and ASR datasets.
}

\label{tab:diversity_crawl}

\centering

\begin{tabular}{ccc}
\hline
\textbf{Dataset name}                       & \textbf{\#Unique tokens} & \textbf{Diversity} \\ \hline
YT-Temporal \cite{zellers_yt-tmp-180m_2021} & 28589                    & 12650              \\
WebVid \cite{bain_webvid_2021}              & 20399                    & 11321              \\ \hline
\end{tabular}

\end{table}

\textbf{Definition.}
Inspired by Nguyen et al. \cite{nguyen_improving_2023}, we define diversity as the number of unique tokens in text annotations:
\begin{equation}
    \textmd{Diversity} = \left | \textmd{Unique tokens in text annotations} \right |. 
\end{equation}
Our definition differs from Nguyen et al. \cite{nguyen_improving_2023} in two aspects. 1) We consider tokens instead of tri-gram since the latter has no specific meanings. 2) We consider only tokens that appeared in all human-annotated video-language datasets. This helps filter some rare long-tailed tokens such as names and special symbols.

\textbf{Calculation.} We randomly sample video-text pairs from Webvid \cite{bain_webvid_2021} and YT-Temporal-180M \cite{zellers_yt-tmp-180m_2021}, and tokenized their text annotations using the same tokenizer as Vicuna 1.5 \cite{vicuna_llm}. We only consider tokens that appeared in human-annotated video-language datasets MSR-VTT \cite{msrvtt_2016}, MSVD \cite{msvd_2011}, VATEX \cite{vatex_2019}, Youcook2 \cite{youcook_2016}, ActivityNet-Captions \cite{anet_caption_2017}, and QV-Highlights \cite{qvhighlights_2021}. We found that the number of unique tokens will saturate with the increase of sampled video-text pairs, as illustrated in \autoref{fig:diversity_scaling}. We fix the number of sampled video-text pairs into 6M. The results of the art assets and ASR datasets are listed in \autoref{tab:diversity_crawl}.

\subsection{Quality}

We define quality as the accuracy of text annotations in the dataset. 
For ASR datasets, we adopt results from the manual evaluation by Miech et al. \cite{miech_howto100m_2019}. It says that 51\% of text annotations have corresponding visual content in the video.
For human-annotated datasets, we set their quality as 100\%.
For art assets datasets, according to the qualitative analysis by Bain et al., the quality can be regarded as 100\% since the text annotations are uploaded by the artists.

\section{Discussion on noise control baselines} \label{sec:sup_noise_baselines}

\begin{figure}[t]
    \centering
    \includegraphics[width=0.495\linewidth]{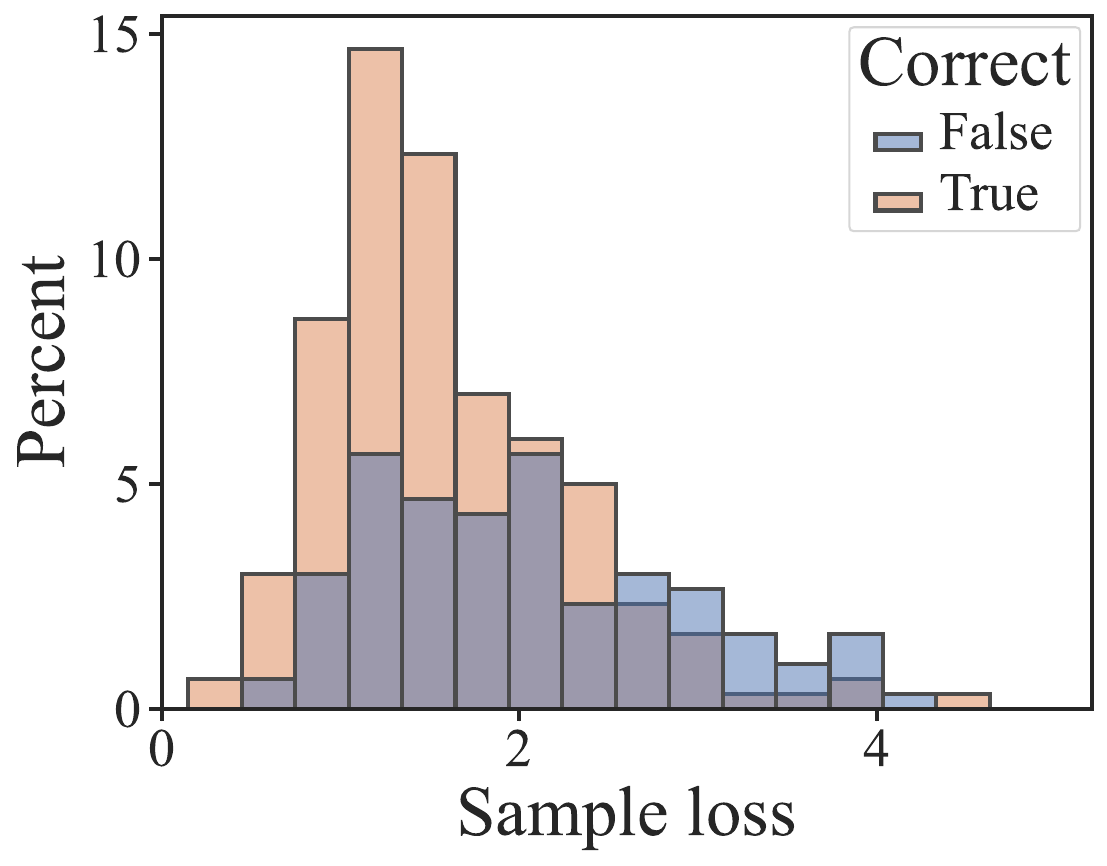}
    \caption{Sample loss of our annotated InternVid \cite{wang_internvid_2023} subset. "Correct" means whether the text annotation can describe the video content. Clearly the sample loss does not follow the mixture of two Gaussian distributions.}
    \label{fig:internvid_sample_loss}
\end{figure}

The noise control targets noise reduction in raw or synthetic annotations. These methods typically make assumptions about noise distribution and mitigate data noise based on these assumptions, such as MIL-NCE \cite{miech_mil_nce_2020}, NCR \cite{huang_ncr_2021}, and CTPR \cite{feng_ctpr_2023}. However, these assumptions may not always align with real data distribution.

MIL-NCE \cite{miech_mil_nce_2020} assumes that video clips and ASR transcripts are just temporally misaligned, and formulates video-language contrastive learning as a multi-instance learning task. However, as revealed by Panda-70M \cite{chen_panda_70m_2024}, the biggest problem with ASR transcripts is that they are weakly aligned with the video, and many of them are irrelevant to the visual content.

NCR \cite{huang_ncr_2021} and CTPR \cite{feng_ctpr_2023} model the loss distribution as a mixture of clean and noisy Gaussian distribution, and re-weight the sample loss by using the posterior probability of clean distribution. Specifically, they assume that the sample loss of all training data is composed of a $K$ component Gaussian Mixture Model:
\begin{equation}
    p(l|\theta) = \sum_{k=1}^K{\beta_k\phi(l|k)},
\end{equation}
where $\beta_k$ and $\phi(l|k)$ are the mixture coefficient and the probability density of the $k$-th component, respectively. For NCR \cite{huang_ncr_2021}, K=2 (clean and noise). For CTPR \cite{feng_ctpr_2023}, K=3 (clean, hard, and noise). To examine this assumption, we manually annotate 300 video-text pairs from the synthetic dataset InternVid \cite{wang_internvid_2023}, and train Video-LLaVA \cite{video_llava} on a randomly sampled 770K subset of InternVid (the annotated data is excluded). The Video-LLaVA is trained exactly as the original paper, except we changed the video-pertaining dataset into Internvid. Then we plot the loss of annotated clean and noisy samples. As illustrated in \autoref{fig:internvid_sample_loss}, the loss distribution does not follow the mixture of Gaussians.

\section{Proof of AdaTaiLr}

\subsection{Preliminaries}


Hölder's inequality. Let $p,q\in[0, \infty]$ with $\frac{1}{p} + \frac{1}{q} = 1$. Then for any vectors $\mathbf{u}, \mathbf{v}$, the following inequality holds: 
\begin{equation} \label{eq:holder}
    \left \| \mathbf{u}\odot\mathbf{v} \right \|_1 \le \left \| \mathbf{u} \right \|_p \left \| \mathbf{v} \right \|_q,
\end{equation}
where $\left \| \cdot \right \|_*$ is the $L_*$-norm, and $\odot$ denotes element-wise vector multiplication.

\subsection{Proof of Theorem~\ref{thm:og}} \label{sec:proof_og}

\og*

\begin{proof}
The estimation error $\epsilon(p_o^{<t}, p_\theta^{<t}, \gamma)$ in \autoref{eq:estimation_error} can be written into a linear function with respect to $\gamma\in[0,1]$:
\begin{equation}
    \epsilon = 
    - \gamma \left [
        \mathcal{L}_{\textmd{TVD}}(p_o^{<t}, p_\theta^{<t}) - 2H_2(p_o^{<t})
    \right ]
    + \mathcal{L}_{\textmd{TVD}}(p_o^{<t}, p_\theta^{<t}).
\end{equation}
When $\mathcal{L}_{\textmd{TVD}}(p_o^{<t}, p_\theta^{<t}) \ge 2H_2(p_o^{<t})$, the above linear function achieves its minimal at $\gamma=1$. When $\mathcal{L}_{\textmd{TVD}}(p_o^{<t}, p_\theta^{<t}) < 2H_2(p_o^{<t})$, it achieves its minimal at $\gamma=0$. 
In conclusion, $\gamma$ achieves its minimum at $\gamma=\Gamma_{opt}(p_o^{<t}, p_\theta^{<t})$.
\end{proof}

\subsection{Proof of Theorem~\ref{thm:aog}} \label{sec:proof_aog}

Before we start the proof, we introduce the basic ideas to approximate the $\Gamma_{opt}$ in \autoref{eq:opt_gamma}, followed by some lemmas. 

As discussed in \autoref{sec:adatailr}, one issue for $\Gamma_{opt}$ is the rough indicator function $\mathbbm{1}[z]$. Therefore, we first use function $f(z)$ as the smooth approximation:
\begin{equation} \label{eq:def_of_f}
    f(z) = \begin{cases}
        0 & \text{ if } z < -\frac{1}{2\lambda} \\
        \lambda z + \frac{1}{2}   & \text{ if } -\frac{1}{2\lambda} \le z \le \frac{1}{2\lambda} \\
        1 & \text{ if } z > \frac{1}{2\lambda}
    \end{cases},
\end{equation}
where $z$ stands for:
\begin{equation} \label{eq:z}
    z=\mathcal{L}_{\textmd{TVD}}(p_o^{<t}, p_\theta^{<t}) - 2H_2(p_o^{<t}),
\end{equation}
and $\lambda>0$ is a constant controlling the smoothness of the approximation. Obviously, we have
\begin{equation}
    \mathbbm{1}[z] = \lim_{\lambda \to \infty} f(z).
\end{equation}

For the other issue that real data distribution $p_o$ which is unavailable during training, we use the predicted distribution $p_\theta$ and the one-hot distribution sampled from real data $e^{(w)}\sim p_o^{<t}$ instead. Specifically, we use $\Tilde{z}$ as the approximation of $z$: 
\begin{equation} \label{eq:appo_z}
    \Tilde{z}=\mathcal{L}_{\textmd{TVD}}(e^{(w)}, p_\theta^{<t}) - 2H_2(p_\theta^{<t}).
\end{equation}

Finally, the optimal $\Gamma_{\textmd{opt}}$ function and our approximation $\Tilde{\Gamma}_{\textmd{opt}}$ can be written as:
\begin{equation}
    \begin{cases}
        \Gamma_{\textmd{opt}} &= \mathbbm{1}[z], \\
        \Tilde{\Gamma}_{\textmd{opt}} &= f(\Tilde{z}).
    \end{cases}
\end{equation}

\begin{lemma}
\label{lem:expection_sample_tvd}
    Given one-hot distribution sampled from real data $e^{(w)}\sim p_o^{<t}$, the expectation of empirical TVD in \autoref{eq:appo_z} is:
    \begin{equation}
        \mathbb{E}_{w\sim p_o^{<t}} \left [
            \mathcal{L}_{\textmd{TVD}}(e^{(w)}, p_\theta^{<t})
        \right ] = 1 - \left \langle p_\theta^{<t}, p_o^{<t} \right \rangle,
    \end{equation}
    where $\left \langle \cdot \right \rangle$ means inner product between vectors.
\end{lemma}

\begin{proof}
    \begin{align}
        &\mathbb{E}_{w\sim p_o^{<t}} \left [
            \mathcal{L}_{\textmd{TVD}}(e^{(w)}, p_\theta^{<t})
        \right ] \\
        &= \mathbb{E}_{w\sim p_o^{<t}} \left [
            1 - \sum_i \textmd{min}\left( e_i^{(w)}, p_{\theta i}^{<t} \right )
        \right ] \\
        &= \mathbb{E}_{w\sim p_o^{<t}} \left [
            1 - p_{\theta w}^{<t}
        \right ] \\
        &= 1 - \sum_i p_{\theta i}^{<t} p_{o i}^{<t} \\
        &= 1 - \left \langle p_\theta^{<t}, p_o^{<t} \right \rangle.
    \end{align}
\end{proof}

\begin{lemma}
\label{lem:l1_and_inf_norm}
    Let $\mathbf{u}, \mathbf{v}\in\mathbb{R}^n$ be two vectors with:
    \begin{equation}
        u_i, v_i \ge 0,
        \sum_i^n{u_i} = 1,
        \sum_i^n{v_i} = 1.
    \end{equation}
    Then the following inequality holds:
    \begin{equation}
        \left \| \mathbf{u} - \mathbf{v} \right \|_\infty
        \le
        \frac{1}{2} \left \| \mathbf{u} - \mathbf{v} \right \|_1.
    \end{equation}
\end{lemma}

\begin{proof}
    Denote $I^+=\{i|u_i - v_i \ge 0\}$ and $I^-=\{i|u_i - v_i < 0\}$.
    \begin{align}
        & \left \| \mathbf{u} - \mathbf{v} \right \|_1 \\
        &= \sum_i{\left | u_i - v_i \right |} \\
        &= \sum_{i\in I^+}{\left(u_i - v_i\right)} + \sum_{i\in I^-}{\left(v_i - u_i\right)} \\
        &= \sum_{i\in I^+}{\left(u_i - v_i\right)} - (1 - \sum_{i\in I^-}{v_i}) + (1 - \sum_{i\in I^-}{u_i}) \\
        &= \sum_{i\in I^+}{\left(u_i - v_i\right)} - \sum_{i\in I^+}{v_i} + \sum_{i\in I^+}{u_i} \\
        &= 2\sum_{i\in I^+}{\left(u_i - v_i\right)}.
    \end{align}
    Symmetrically, we can also get:
    \begin{equation}
        \left \| \mathbf{u} - \mathbf{v} \right \|_1 = 2\sum_{i\in I^-}{\left(v_i - u_i\right)}.
    \end{equation}
    Therefore, 
    \begin{align}
        & \left \| \mathbf{u} - \mathbf{v} \right \|_1 \\
        &\ge 2 \max_{i}{\left|u_i - v_i\right|} \\
        &= 2 \left \|u_i - v_i\right \|_\infty.
    \end{align}
    That is
    \begin{equation}
        \left \| \mathbf{u} - \mathbf{v} \right \|_\infty
        \le
        \frac{1}{2} \left \| \mathbf{u} - \mathbf{v} \right \|_1.
    \end{equation}
\end{proof}

\begin{lemma}
    \label{lem:diff_of_z}
    Assume that after some warm-up steps during training, there exists $D>0$ under which $\left \| p_\theta - p_o \right \|_1 \le 2D$.
    Given one-hot distribution sampled from real data $e^{(w)}\sim p_o^{<t}$, the distance between $z$ and $\Tilde{z}$ can be characterized by the subsequent bound:
    \begin{equation}
         \left | z - \mathbb{E}_{w\sim p_o} \left [ \Tilde{z} \right ] \right | \le 4D.
    \end{equation}
\end{lemma}

\begin{proof}
    Firstly, by expanding $z$ in \autoref{eq:z} and $\Tilde{z}$ in \autoref{eq:appo_z} we can get:
    \begin{align}
        z &= \mathcal{L}_{\textmd{TVD}}(p_o^{<t}, p_\theta^{<t}) - 2H_2(p_o^{<t}) \\
          &= \frac{1}{2} \left \| p_\theta^{<t} - p_o^{<t} \right \|_1 - \left ( 1 - \left \| p_o^{<t} \right \|_2^2 \right ),
    \end{align}
    and
    \begin{align}
        \mathbb{E}_{w\sim p_o} \left [ \Tilde{z} \right ] &= \mathbb{E}_{w\sim p_o} \left [ \mathcal{L}_{\textmd{TVD}}(e^{(w)}, p_\theta^{<t})  \right ] - 2H_2(p_\theta^{<t}) \\
        &= 1 - \left \langle p_\theta^{<t}, p_o^{<t} \right \rangle - 2 \left ( 1 - \left \| p_\theta^{<t} \right \|_2^2 \right ), \label{eq:diff_of_z_proof_appo_z_wo_E}
    \end{align}
    where \autoref{eq:diff_of_z_proof_appo_z_wo_E} uses \autoref{lem:expection_sample_tvd}.

    Then, $\left | z - \Tilde{z} \right |$ can be represented as:
    \begin{equation} \label{eq:diff_of_z_proof_abc}
        \left | z - \Tilde{z} \right | = \left | A + B + C \right |,
    \end{equation}
    where:
    \begin{equation}
        \begin{cases}
         A = \frac{1}{2} \left \| p_\theta^{<t} - p_o^{<t} \right \|_1 \\
         B = \left \langle p_\theta^{<t}, p_o^{<t} \right \rangle - \left \| p_\theta^{<t} \right \|_2^2 \\
         C = \left \| p_o^{<t} \right \|_2^2 - \left \| p_\theta^{<t} \right \|_2^2 
        \end{cases}.
    \end{equation}

    Afterwards, we can derive the bounds for $A$, $B$, and $C$, respectively.
    \begin{equation}
         \left | A \right | = \frac{1}{2} \left \| p_\theta^{<t} - p_o^{<t} \right \|_1  \le D.
    \end{equation}
    \begin{align}
        \left | B \right | &= \left | \left \langle p_\theta^{<t}, p_o^{<t} \right \rangle - \left \| p_\theta^{<t} \right \|_2^2 \right | \\
                           &= \left | \sum_i p_{\theta i}^{<t}p_{o i}^{<t} - \sum_i p_{\theta i}^{<t}p_{\theta i}^{<t} \right | \\
                           &= \left | \sum_i p_{\theta i}^{<t} \left( p_{o i}^{<t} - p_{\theta i}^{<t} \right ) \right | \\
                           &\le \sum_i \left |p_{\theta i}^{<t} \left( p_{o i}^{<t} - p_{\theta i}^{<t} \right ) \right | \\
                           &= \left \| p_\theta^{<t} \odot \left( p_o^{<t} - p_\theta^{<t} \right ) \right \|_1 \\
                           &\le \left \| p_\theta^{<t} \right \|_1 \left \| p_o^{<t} - p_\theta^{<t} \right \|_\infty \label{eq:bound_B_holder_after} \\
                           &= \left \| p_o^{<t} - p_\theta^{<t} \right \|_\infty \\
                           &\le \frac{1}{2} \left \| p_o^{<t} - p_\theta^{<t} \right \|_1 \label{eq:bound_B_norm_1} \\
                           &\le D,
    \end{align}
    where \autoref{eq:bound_B_holder_after} uses Hölder's inequality in \autoref{eq:holder}, and \autoref{eq:bound_B_norm_1} uses the conclusion from \autoref{lem:l1_and_inf_norm}.
    \begin{align}
        \left | C \right | &= \left | \left \| p_o^{<t} \right \|_2^2 - \left \| p_\theta^{<t} \right \|_2^2 \right | \\
                           &= \left | \sum_i 
                                \left ( p_{\theta i}^{<t} - p_{o i}^{<t} \right ) 
                                \left ( p_{\theta i}^{<t} + p_{o i}^{<t} \right ) 
                            \right | \\
                           &\le \sum_i \left | 
                                \left ( p_{\theta i}^{<t} - p_{o i}^{<t} \right ) 
                                \left ( p_{\theta i}^{<t} + p_{o i}^{<t} \right ) 
                            \right | \\
                           &= \left \| \left ( p_\theta^{<t} - p_o^{<t} \right ) 
                                \odot 
                                \left ( p_\theta^{<t} + p_o^{<t} \right ) \right \|_1 \\
                           &\le \left \| p_\theta^{<t} - p_o^{<t} \right \|_\infty
                                \left \| p_\theta^{<t} + p_o^{<t} \right \|_1 \label{eq:bound_C_after_holder} \\
                           &= 2 \left \| p_\theta^{<t} - p_o^{<t} \right \|_\infty   \\
                           &\le \left \| p_\theta^{<t} - p_o^{<t} \right \|_1 \label{eq:bound_C_norm_1} \\
                           &\le 2D,
    \end{align}
    where \autoref{eq:bound_C_after_holder} uses Hölder's inequality in \autoref{eq:holder}, and \autoref{eq:bound_C_norm_1} uses the conclusion from \autoref{lem:l1_and_inf_norm}.
    
    Finally, by combining \autoref{eq:diff_of_z_proof_abc} and the above bounds for $A$, $B$, and $C$, we can prove the Lemma:
    \begin{align}
        \left | z - \Tilde{z} \right | &= \left | A + B + C \right | \\
                                       &\le \left | A \right | + \left | B \right | + \left | C \right | \\
                                       &\le 4D.
    \end{align}
\end{proof}

\begin{lemma}[Error from Smooth Approximation]
\label{lem:error_smooth}
    When employing function $f(z)$ as a smooth approximation to the indicator function $\mathbbm{1}[z]$, the approximation error can be characterized by the subsequent bound:
    \begin{equation}
        \left [ \mathbbm{1}[z] - f(z) \right ] z \le \frac{1}{16\lambda}.
    \end{equation}
\end{lemma}

\begin{proof}
    \begin{align}
        \left [ \mathbbm{1}[z] - f(z) \right ] z &= \begin{cases}
            0 & \text{ if } z > \frac{1}{2\lambda} \\
            \left ( \frac{1}{2} - \lambda z \right ) z & \text{ if } 0 \le z \le \frac{1}{2\lambda} \\
            - \left ( \frac{1}{2} + \lambda z \right ) z & \text{ if } -\frac{1}{2\lambda} \le z \le 0 \\
            0 & \text{ if } z < -\frac{1}{2\lambda}
        \end{cases} \\
        &= \begin{cases}
            0 & \text{ if } |z| > \frac{1}{2\lambda} \\
            \left ( \frac{1}{2} - \lambda |z| \right ) |z| & \text{ if } |z| \le \frac{1}{2\lambda} \\
        \end{cases} \\
        &\le \left ( \frac{1}{2} - \lambda |z| \right ) |z| \\
        &= -\lambda \left ( |z| - \frac{1}{4\lambda} \right )^2 + \frac{1}{16\lambda} \\
        &\le \frac{1}{16\lambda}.
    \end{align}
\end{proof}

\begin{lemma}[Error from Data Distribution Approximation]
    \label{lem:error_distribution}
    Assume that after some warm-up steps during training, there exists $D>0$ under which $\left \| p_\theta - p_o \right \|_1 \le 2D$.
    Given one-hot distribution sampled from real data $e^{(w)}\sim p_o^{<t}$, the error can be characterized by the subsequent bound:
    \begin{equation}
         \mathbb{E}_{w\sim p_o} \left [ \left ( f(z) - f(\Tilde{z}) \right ) z \right ] = \frac{a}{\lambda} + b D,
    \end{equation}
    where $a,b$ is constant depending on the relationship of $\lambda$, $D$, $|a|<\frac{1}{2}$ and $|b|<4$.
\end{lemma}

\begin{proof}
    Since $w$ only exists in $\Tilde{z}$, and $f(\cdot)$ defined in \autoref{eq:def_of_f} is a linear function:
    \begin{align}
        & \mathbb{E}_{w\sim p_o} \left [ \left [f(z) - f(\Tilde{z}) \right ] z  \right ] \\
        & = \left [ f(z) - \mathbb{E}_{w\sim p_o} \left [ f(\Tilde{z}) \right ] \right ] z, \\
        & = \left [ f(z) - f( \mathbb{E}_{w\sim p_o} \left [\Tilde{z} \right ] ) \right ] z. \label{eq:proof_error_distribution_target}
    \end{align} 

    In order to find the upper bound of \autoref{eq:proof_error_distribution_target}, we can adjust the free variable $\mathbb{E}_{w\sim p_o} \left [\Tilde{z} \right ]$ to make $\left | f(z) - f( \mathbb{E}_{w\sim p_o} \left [\Tilde{z} \right ] ) \right | $ as large as possible. Considering the conclusion that $\left | z - \mathbb{E}_{w\sim p_o} \left [ \Tilde{z} \right ] \right | \le 4D$ in \autoref{lem:diff_of_z}, we let $\mathbb{E}_{w\sim p_o} \left [ \Tilde{z} \right ] = z \pm 4D$ such that:
    \begin{equation}
        \left [ f(z) - f( \mathbb{E}_{w\sim p_o} \left [\Tilde{z} \right ] ) \right ] z \le \left [ f(z) - f( z \pm 4D ) \right ] z.
    \end{equation}

    Denote $U = \left [ f(z) - f( z - 4D ) \right ] z$ and $V = \left [ f(z) - f( z + 4D ) \right ] z$. Since $f(z)$ is a piecewise function, the upper bound of \autoref{eq:proof_error_distribution_target} can be found in the maximums of $U, V$ and breakpoints of $f(z)$.
    
    To find the maximums of $U$, we calculate the derivative of $U$ is:
    \begin{align}
        \frac{\mathrm{d} U}{\mathrm{d} z} &= 
        \left [ \frac{\mathrm{d} f(z)}{\mathrm{d} z} - \frac{\mathrm{d} f(z - 4D)}{\mathrm{d} z} \right ]z
        +
        f(z) - f(z - 4D), \\
        \frac{\mathrm{d}^2 U}{\mathrm{d} z^2} &= 
        \frac{\mathrm{d} f(z)}{\mathrm{d} z} - \frac{\mathrm{d} f(z - 4D)}{\mathrm{d} z}.
    \end{align}
    Let $\frac{\mathrm{d} U}{\mathrm{d} z}=0$ and $\frac{\mathrm{d} U^2}{\mathrm{d} z^2}<0$, the maximum point is:
    \begin{equation}
        z_{U_\textmd{max}} = \frac{1}{4\lambda} + 2D.
    \end{equation}
    \begin{align}
        U_{\textmd{max}} &= \left [
            f(z_{U_\textmd{max}}) - f(z_{U_\textmd{max}} - 4D)
        \right ] z_{U_\textmd{max}}, \\
        &\le (1 - 0) z_{U_\textmd{max}}, \\
        &= \frac{1}{4\lambda} + 2D. \label{eq:proof_error_distribution_Umax}
    \end{align}

    To find the maximums of $V$, we calculate the derivative of $V$ is:
    \begin{align}
        \frac{\mathrm{d} V}{\mathrm{d} z} &= 
        \left [ \frac{\mathrm{d} f(z)}{\mathrm{d} z} - \frac{\mathrm{d} f(z + 4D)}{\mathrm{d} z} \right ]z
        +
        f(z) - f(z + 4D), \\
        \frac{\mathrm{d}^2 V}{\mathrm{d} z^2} &= 
        \frac{\mathrm{d} f(z)}{\mathrm{d} z} - \frac{\mathrm{d} f(z + 4D)}{\mathrm{d} z}. \label{eq:proof_error_distribution_Vmax}
    \end{align}
    Let $\frac{\mathrm{d} V}{\mathrm{d} z}=0$ and $\frac{\mathrm{d} V^2}{\mathrm{d} z^2}<0$, the maximum point is:
    \begin{equation}
        z_{V_\textmd{max}} = - \frac{1}{4\lambda} - 2D.
    \end{equation}
    \begin{align}
        V_{\textmd{max}} &= \left [
            f(z_{V_\textmd{max}}) - f(z_{V_\textmd{max}} + 4D)
        \right ] z_{U_\textmd{max}}, \\
        &\le (0 - 1) z_{V_\textmd{max}}, \\
        &= \frac{1}{4\lambda} + 2D.
    \end{align}

    When $z$ at breakpoints:
    \begin{equation}
        \left [ f(z) - f( \mathbb{E}_{w\sim p_o} \left [\Tilde{z} \right ] ) \right ] z \le \frac{1}{2\lambda}. \label{eq:proof_error_distribution_z_break}
    \end{equation}

    When $z\pm 4D$ at breakpoints:
    \begin{equation}
        \left [ f(z) - f( \mathbb{E}_{w\sim p_o} \left [\Tilde{z} \right ] ) \right ] z \le \left | \frac{1}{2\lambda} - 4D \right |. \label{eq:proof_error_distribution_z4d_break}
    \end{equation}

    Finally, according to \autoref{eq:proof_error_distribution_Umax}, \autoref{eq:proof_error_distribution_Vmax}, \autoref{eq:proof_error_distribution_z_break}, and \autoref{eq:proof_error_distribution_z4d_break}:
    \begin{equation}
        \mathbb{E}_{w\sim p_o} \left [ \left ( f(z) - f(\Tilde{z}) \right ) z \right ] = \frac{a}{\lambda} + b D,
    \end{equation}
    where $a,b$ is constant depending on the relationship of $\lambda$, $D$, $|a|<\frac{1}{2}$ and $|b|<4$.

\end{proof}

\aog*

\begin{proof}
Since $\Gamma_{\textmd{opt}}$ is the optimal value that minimize $\epsilon$, for any $\tilde{\Gamma}_\textmd{opt}$ there must be:
\begin{equation}
    \epsilon(p_o^{<t}, p_\theta^{<t}, \tilde{\Gamma}_\textmd{opt}) - \epsilon(p_o^{<t}, p_\theta^{<t}, \Gamma_\textmd{opt}) \ge 0,
\end{equation}
and the expected approximation error can be written as:
\begin{align} 
    &\mathbb{E}_{w\sim p_o} \left [ \left |  
        \epsilon(p_o^{<t}, p_\theta^{<t}, \tilde{\Gamma}_{opt}) - \epsilon(p_o^{<t}, p_\theta^{<t}, \Gamma_{opt})
    \right | \right ] \\
    &= \mathbb{E}_{w\sim p_o} \left [ 
        \epsilon(p_o^{<t}, p_\theta^{<t}, \tilde{\Gamma}_{opt}) - \epsilon(p_o^{<t}, p_\theta^{<t}, \Gamma_{opt})
    \right ] \\
    &= \mathbb{E}_{w\sim p_o} \left [  
        \left ( {\Gamma_{opt} - \Tilde{\Gamma}_{opt}} \right ) z
    \right ] \\
    &= \mathbb{E}_{w\sim p_o} \left [
        \left ( \mathbbm{1}[z] - f(\Tilde{z}) \right ) z 
    \right ] \\
    &= \mathbb{E}_{w\sim p_o} \left [
        \left ( \mathbbm{1}[z] - f(z) + f(z) - f(\Tilde{z}) \right ) z
    \right ] \\
    &= \mathbb{E}_{w\sim p_o} \left [ 
        \left ( \mathbbm{1}[z] - f(z) \right )  z
        + 
        \left ( f(z) - f(\Tilde{z}) \right ) z 
    \right ] \\
    &= \mathbb{E}_{w\sim p_o} \left [ 
        \left ( \mathbbm{1}[z] - f(z) \right )  z \right ]
        + 
        \mathbb{E}_{w\sim p_o} \left [ 
        \left ( f(z) - f(\Tilde{z}) \right ) z 
    \right ] \\
    &\le \frac{a}{\lambda} +bD.
\end{align}
In the last step, we adopt \autoref{lem:error_smooth} and \autoref{lem:error_distribution}.

\end{proof}




\end{document}